\documentclass[nohyperref]{article}

\usepackage{layouts}
\usepackage{microtype}
\usepackage{graphicx}
\usepackage{caption}
\usepackage{subcaption}
\usepackage{booktabs} 
\usepackage[dvipsnames]{xcolor}
\usepackage{hyperref}

\usepackage{cancel}

\usepackage[ruled, vlined, linesnumbered,algo2e]{algorithm2e}

\SetKwInput{KwData}{Inputs}
\SetKwInput{KwResult}{Output}

\usepackage[accepted]{icml2022}

\usepackage{float}

\usepackage[dvipsnames]{xcolor}

\usepackage{amsmath}
\usepackage{amssymb}
\usepackage{mathtools}
\usepackage{amsthm}
\usepackage{bm, bbm}

\usepackage{nth}
\usepackage{nicefrac}

\usepackage[textsize=tiny]{todonotes}

\usepackage{cleveref} 
\crefformat{equation}{(#2#1#3)}
\crefformat{figure}{Figure #2#1#3}
\crefformat{table}{Table #2#1#3}
\crefformat{appendix}{Appendix #2#1#3}
\crefformat{section}{Section #2#1#3}

\usepackage{multirow}
\usepackage{caption}

\usepackage{enumitem}
\theoremstyle{plain}
\newtheorem{theorem}{Theorem}
\newtheorem{proposition}[theorem]{Proposition}
\newtheorem{lemma}[theorem]{Lemma}

\newtheorem{recommendation}{Recommendation}

\theoremstyle{definition}
\newtheorem{definition}{Definition}

\theoremstyle{remark}

\newcommand{\R}{\mathbb{R}}

\renewcommand{\det}[0]{\text{det}}

\DeclareMathOperator*{\argmin}{argmin}
\DeclareMathOperator*{\argmax}{argmax}

\newcommand{\mcL}{\mathcal{L}}

\newcommand{\mcN}{\mathcal{N}}
\newcommand{\mcX}{\mathcal{X}}
\newcommand{\mcY}{\mathcal{Y}}
\newcommand{\mcG}{\mathcal{G}}
\newcommand{\mcM}{\mathcal{M}}
\newcommand{\mcF}{\mathcal{F}}

\newcommand{\norm}[1]{\|#1\|}

\newcommand{\spaced}[1]{\quad \text{#1} \quad}

\newcommand{\RN}[1]{%
  \textup{\uppercase\expandafter{\romannumeral#1}}%
}

\icmltitlerunning{Adapting the Linearised Laplace Model Evidence for Modern Deep Learning}

\begin{document}

\twocolumn[
\icmltitle{Adapting the Linearised Laplace Model Evidence for Modern Deep Learning} 

\icmlsetsymbol{equal}{*}

\begin{icmlauthorlist}
\icmlauthor{Javier Antorán}{cam}
\icmlauthor{David Janz}{equal,alberta}
\icmlauthor{James Urquhart Allingham}{equal,cam}
\icmlauthor{Erik Daxberger }{cam,mpi}
\icmlauthor{Riccardo Barbano}{ucl}
\icmlauthor{Eric Nalisnick}{uva}
\icmlauthor{José Miguel Hernández-Lobato}{cam}
\end{icmlauthorlist}

\icmlaffiliation{cam}{University of Cambridge}
\icmlaffiliation{alberta}{University of Alberta}
\icmlaffiliation{mpi}{Max Planck Institute for Intelligent Systems, Tübingen}
\icmlaffiliation{ucl}{ University College London}
\icmlaffiliation{uva}{University of Amsterdam}

\icmlcorrespondingauthor{Javier Antorán}{ja666@cam.ac.uk}

\icmlkeywords{Machine Learning, ICML, Laplace approximation, linearised Laplace, neural tangent kernel, NTK, BNNs, Bayesian neural networks, Bayesian inference, model selection, uncertainty estimation}
\vskip 0.3in
]


\printAffiliationsAndNotice{\icmlEqualContribution} 

\begin{abstract}
    The linearised Laplace method for estimating model uncertainty has received renewed attention in the Bayesian deep learning community.
    The method provides reliable error bars and admits a closed-form expression for the model evidence, allowing for scalable selection of model hyperparameters. 
    In this work, we examine the assumptions behind this method, particularly in conjunction with model selection.
    We show that these interact poorly with some now-standard tools of deep learning---stochastic approximation methods and normalisation layers---and make recommendations for how to better adapt this classic method to the modern setting.
    We provide theoretical support for our recommendations and validate them empirically on MLPs, classic CNNs, residual networks with and without normalisation layers, generative autoencoders~and~transformers.
\end{abstract}

\section{Introduction}
\label{sec:intro}


Model selection and uncertainty estimation are two important open problems in deep learning. The former aims to select network hyperparameters and architectures without costly cross-validation \citep{Mackay1992Thesis,Immer2021Marginal}. The latter provides a measure of fidelity of network predictions that can be used in downstream tasks such as experimental design \citep{barbano2022design}, sequential decision making \citep{Janz19successor} and in safety-critical settings \citep{fridman2019arguing}.
In this paper, we consider a classical approach to these two problems: the linearised Laplace method \citep{Mackay1992Thesis}, which has recently been shown to be one of the best performing methods for inference in neural networks \citep{Khan2019Approximate,Kristiadi2020being,Immer2021Improving,daxberger2021bayesian,daxberger2021laplace}.

Linearised Laplace approximates the output of a neural network (NN) with a first order Taylor expansion (a linearisation) around optimal NN parameters. It then uses standard linear-model-type error bars to approximate the uncertainty in the output of the NN, while retaining the NN point-estimate as the predictive mean.
The latter feature means that, unlike other Bayesian deep learning procedures, the linearised Laplace uncertainty estimates do not come at the cost of the accuracy of the predictive mean  \citep{snoek2019can,ashukha2020pitfalls,antoran2020depth}. A downside of the method is that its uncertainty estimates are very sensitive to the choice of certain (prior) hyperparameters \citep{daxberger2021bayesian}. Our work looks at the model evidence maximisation method for choosing these hyperparameters, as used in the seminal work of \citet{Mackay1992Thesis} and later by \citet{Immer2021Marginal}. In contrast with often used cross-validation, evidence maximisation reduces model selection to an (often convex) optimisation problem, and can scale to a large number of hyperparameters.


Our contributions are the identification of certain incompatibilities between the assumptions underlying the classical linearised Laplace model evidence and modern deep learning methodology, and a number of recommendations on how to adapt the method in light of these. In particular:
\begin{itemize}[topsep=0pt]
    \item A core assumption of linearised Laplace is that the point of linearisation is a minimum of the training loss. When the neural network is not trained to convergence (and this is almost never done), this does not hold and results in
    severe deterioration of the model evidence estimate. We show that this can be corrected by instead considering the optima of the linearised model's loss, that is solving a convex optimisation problem.    
    \item We show that for networks with normalisation layers (such as batch norm \citep{ioffe2015batch}), the linearised Laplace predictive distribution can fail to be well-defined. However, this can be resolved by separately parametrising the prior corresponding to normalised and non-normalised network parameters. 
\end{itemize}
We provide both theoretical and empirical justification for both points above. The resulting recommended procedure significantly outperforms a na\"ive implementation on a series of standard tasks and a wide range of neural architectures.

\section{The linearised Laplace method}\label{sec:preliminaries}

We begin by reviewing a near-textbook version of the linearised Laplace method for the tractable approximation of uncertainty estimates for neural networks, with model evidence maximisation for hyperparamter selection \citep{MacKay2011Information,bishop2006pattern,Immer2021Improving,Immer2021Marginal}. We focus on the motivation and assumptions behind this method.

\paragraph{Notation} For a vector $v$ and a positive semidefinite matrix $M$ of compatible dimensions, $\|v\|_M$ denotes the norm $\sqrt{v^T M v}$ and we write $\|v\| = \|v\|_I$. We denote by $\cdot$ a vector-matrix or matrix-matrix product where this may help with clarity. For $m$ a positive integer and $f$ a function with parameters $p$, $\partial^m_p [f(p)](\tilde{p})$ denotes the $m$th order mixed partial derivatives of $f$ with respect to $p$ evaluated at $\tilde{p}$. We use $\partial^m_p f(\tilde{p})$ where no ambiguity exists. We assume such partial derivatives are well defined pointwise where~needed.

\paragraph{Model} We consider performing regression or classification using a neural network of the form $\mu(f(\theta,x))$, where $f \colon \Theta \times \mcX \mapsto \mcY$ is a neural network with parameter space $\Theta$, input space $\mcX$, output space $\mcY$ and $\mu$ is a linking function. We take $\mcX, \mcY, \Theta$ to be subsets of finite dimensional Euclidean spaces. We assume that we train the network $f$ by minimising a loss function of the form
\begin{equation}\label{eq:optobj}
    \mcL_{f,\Lambda}(\theta) = L(f(\theta,\cdot)) + \|\theta\|^2_\Lambda,
\end{equation}
where $L \colon \mcY^{\Theta\times\mcX} \mapsto \R_+$ gives the data-fit and is of the form $L(f(\theta, \cdot)) = \sum_i \ell(f(\theta, x_i), y_i)$ for some negative log-likelihood function $\ell$. We assume $L$ is a proper convex function. The term $\|\theta\|^2_\Lambda$ is a regulariser corresponding to the log density of a Gaussian prior over $\theta$. We take $\Lambda$, a positive-definite matrix, to be a model hyperparameter. 

The loss minimisation viewpoint set out in \cref{eq:optobj} is equivalent to finding the \emph{maximum a~posteriori} (MAP) solution to a posterior $P$ for $\theta$ defined by the Lebesgue density
\begin{equation}
    \frac{dP}{d\nu}(\theta) = \frac{1}{Z}e^{-\mcL_{f,\Lambda}(\theta)},
\end{equation}
for $Z>0$ a normalisation constant. Our aim will be to estimate the posterior distribution for $\mu(f(\theta, x))$ when $\theta \sim P$. For ease of exposition, we focus the posterior for $f$ throughout; the push-forward by $\mu$ adds no further complexity. 

\paragraph{Approximation} The method assumes that we have some parameters $\tilde\theta \in \Theta$ obtained by optimising the neural network loss $\mcL_{f,\Lambda}$ for some initial choice of $\Lambda$. For this section, we assume that $\tilde\theta$ is a minimum of $\mcL_{f,\Lambda}$. We proceed to apply approximations to $f$ and $P$ based on consecutive Taylor expansions about $\tilde\theta$.

First, we approximate $f$ with the affine function
\begin{equation}\label{eq:h-def}
    h(\theta, x) = f(\tilde\theta, x) + \partial_\theta f (\tilde\theta, x) \cdot ( \theta - \tilde\theta ),
\end{equation}
a first order approximation to $f$ about $\tilde\theta$. We then approximate the loss function for the linearised model, $\mcL_{h, \Lambda}(\theta) = L(h(\theta, \cdot)) + \|\theta\|^2_\Lambda$, with a second order Taylor expansion about $\tilde\theta$. Since $\partial_\theta h (\tilde\theta, \cdot) \;{=}\; \partial_\theta f(\tilde\theta, \cdot)$ and $\tilde\theta {\in} \argmin_{\theta}\mcL_{f,\Lambda}$, we have that $\partial_\theta \mcL_{h, \Lambda}(\tilde\theta) \:{=}\: \partial_\theta \mcL_{f, \Lambda}(\tilde\theta) \:{=}\: 0$, and thus the first order term vanishes. This leaves us with the approximation
\vspace{-0.07cm}
\begin{equation}\label{eq:G_function}
    \mcG_{\tilde\theta}(\theta) = \mcL_{h, \Lambda}(\tilde\theta) + \frac{1}{2}\| \theta-\tilde\theta\|^{2}_{\partial^2_\theta \mcL_{h, \Lambda}(\tilde\theta)}.
    \vspace{-0.05cm}
\end{equation}
We use the approximate loss $\mcG_{\tilde\theta}$ to define an approximate posterior $Q$ by taking its Lebesgue density to be proportional to $e^{-\mcG_{\tilde\theta}(\theta)}$. By inspection, $Q$ is then Gaussian with mean $\tilde\theta$ and covariance $(\partial^2_\theta \mcL_{h, \Lambda}(\tilde\theta))^{-1}$. The approximate predictive posterior is given by $h(\theta, \cdot)$, $\theta \sim Q$. Since $h$ is affine, this is again Gaussian. Writing the Hessian of $\mcL_{h, \Lambda}$ at $\tilde\theta$ as
\vspace{-0.05cm}
\begin{equation}
    \smash{\partial^2_\theta \mcL_{h, \Lambda}(\tilde\theta) = H + \Lambda} \,\,\,\,\text{with}\,\,\,\, H = \partial^2_\theta [L(h(\theta, \cdot))](\tilde\theta),
    \vspace{-0.05cm}
\end{equation}
and $J(\cdot) = \partial_\theta f(\tilde\theta, \cdot)$ for the Jacobian of $f$ at $\tilde\theta$, we can write the approximate predictive posterior distribution as
\vspace{-0.05cm}
\begin{equation}\label{eq:h-posterior}
    h(\theta, \cdot) \sim \mcN\left(
    f(\tilde\theta, \cdot),\ 
    \| J(\cdot)\|^2_{(H + \Lambda)^{-1}}\right).
    \vspace{-0.05cm}
\end{equation}
We have thus augmented the neural network mean with Gaussian error bars obtained from a linear model based on a feature expansion given by the Jacobian of the network.

\paragraph{Model selection} The predictive posterior $h(\theta, \cdot)$, $\theta \sim Q$ given in \cref{eq:h-posterior} has an explicit dependence on $\Lambda$.  This parameter significantly affects the predictive posterior variance, but we have no method for choosing it \textit{a~priori}. We instead follow an empirical-Bayes procedure: we interpret $\Lambda$ as the precision of a prior $Q_{0} = \mcN(0, \Lambda)$ and choose $\Lambda$ as that most likely to generate the observed data given the model $h(\theta, \cdot)$ with $\theta \sim Q_0$. This yields the (log) objective 
\begin{equation}
    \mcM_{\tilde\theta}(\Lambda) = 
    \log
    \int_\Theta e^{-\mcG_{\tilde\theta}(\theta)} d\nu,
\end{equation}
called the model evidence, with $\nu$ the Lebesgue measure. Here, $\mcM_{\tilde\theta}(\Lambda)$ depends on $\Lambda$ implicitly through $\mcG_{\tilde\theta}$. Our approximation $\mcG_{\tilde\theta}$ gives a tractable expression for $\mcM_{\tilde\theta}$,
\vspace{-0.05cm}
\begin{equation}\label{eq:model-evidence}
    \mcM_{\tilde\theta}(\Lambda) = -\frac{1}{2}\left[ \|\tilde\theta\|^{2}_{\Lambda}  + \log \det(\Lambda^{-1}H + I)\right] + C,
    \vspace{-0.05cm}
\end{equation}
where $C$ is independent of $\Lambda$. Throughout, we will constrain $\Lambda$ to the set of positive diagonal matrices, as in \citet{Mackay1992Thesis}. Maximising $\mcM_{\tilde\theta}$ is a concave optimisation problem.

\paragraph{Discussion} 
We made a number of assumptions in our derivation. First, that the data-fit term $L$ is convex. This is satisfied by the standard losses used to train neural networks. We also assumed that the posterior over weights $P$ is sharply peaked around its optima such that it can be approximated well by a quadratic expansion and that $h$ is a good approximation to $f$ near the linearisation point. These assumptions we do not question further. We made one further important assumption, that the linearisation point $\tilde \theta$ is a local minimum of $\mcL_{f, \Lambda}$ and thus it is also a  minima of the linearised loss $\mcL_{h, \Lambda}$. This final assumption will be the focus of our work.

Our overview of the linearised Laplace method diverged from the original presentation in one significant manner. \citet{Mackay1992Thesis} presents the method as an Expectation Maximisation (EM) procedure: 
in the E step, we perform optimisation to find $\tilde \theta$ and compute the approximate posterior $Q$ using this $\tilde\theta$ as in \cref{eq:G_function}; in the $M$ step, the prior precision $\Lambda$ is optimised such that the model evidence is maximised \cref{eq:model-evidence}. 
The E and M steps are iterated until convergence. 
In modern settings, retraining the network (that is, performing the E step more than once) is too expensive. A single step of the EM procedure (as presented) is used instead.
The posterior predictive mean is set to match $f(\tilde\theta, \cdot)$, ignoring that $\tilde\theta$ will no longer be a mode of $\mathcal{L}_{f,\Lambda}$ after an update of $\Lambda$.
This choice keeps the NN's predictions unchanged, and is considered an advantage of linearised Laplace over competing Bayesian deep learning~methods.

\section{Overview of results}\label{sec:main_results}

The linearised Laplace method with model-evidence maximisation procedure, as described, is mostly as first introduced by \citet{Mackay1992Thesis}. Since then, deep learning training procedures and architectures have changed. 
Stochastic first order methods are used to minimise the loss function in place of the second order full-batch methods common in classical literature \citep{lecun1996effiicient,Amari2000info}. We  often do not use a low value of the loss $\mcL_{f, \Lambda}$ as a stopping criterion, but instead monitor some separate validation metric. Also, normalisation layers are ubiquitous. 

Since the derivations of \cref{sec:preliminaries} assume that we linearise $f$ and expand $\mcL_{h, \Lambda}$ about a local minimum of $\mcL_{f, \Lambda}$ (and thus of $\mcL_{h, \Lambda}$), modern practises pose difficulties for the presented method. The rest of this section explores these issues in turn, proposes a modern adaptation of the linearised Laplace method, and discusses some interesting special cases. 

\subsection{An alternative adaptation of linearised Laplace}\label{sec:adapt-laplace}

We consider a na\"ive implementation of the linearised Laplace method in the context of modern neural networks as that using the linearisation point $\tilde\theta$ in the expression for the model evidence $\mcM_{\tilde \theta}$ \cref{eq:model-evidence}, even if this point is known to not be a local minimum of $\mcL_{f, \Lambda}$. 
We now propose an~alternative. 

We begin, as before, by linearising $f$ about  $\tilde\theta$, the point returned by (possibly stochastic or incomplete) optimisation of the neural network loss $\mcL_{f, \Lambda}$ and constructing the feature expansion $J\colon x \mapsto \partial_\theta f(\tilde\theta, \cdot)$. Under broad assumptions discussed in \cref{sec:dense-final-layer}, this choice means the posterior mean $f(\tilde\theta, \cdot)$ is contained within the linear span of the Jacobian features $J(\cdot)$. This yields credence to the interpretation of the resulting error bars as uncertainty about the NN output.

We diverge from \cref{sec:preliminaries} in how we approximate $\mcL_{h, \Lambda}$. We start by noting that since $\tilde\theta$ is not a local minimum of $\mcL_{f, \Lambda}$, it is not one of $\mcL_{h, \Lambda}$ either (\cref{proposition:f-not-then-h-not}). Thus, $\tilde\theta$ is not a suitable point for a quadratic approximation to $\mcL_{h,\Lambda}$. However, for any given $\Lambda$, the loss for the linearised model $\mcL_{h, \Lambda}$ is a convex function of $\theta$ ($L$ is convex and $h$ is linear in $\theta$), and thus has a well defined minimiser; expanding the loss about this minimiser will yield a more faithful approximation. 
Moreover, for each fixed $\theta$, $\mcM_{\theta}(\Lambda)$ is concave in $\Lambda$, yielding a maximiser. Iteratively minimising the convex $\mcL_{h,\Lambda}(\theta)$ and maximising the concave $\mcM_{\theta}(\Lambda)$ yields a simultaneous stationary point $(\theta_\star, \Lambda_\star)$ satisfying  
\begin{equation*}
    \textstyle{\theta_\star \!\in \argmin_\theta \mcL_{h,\Lambda_\star}(\theta) \!\!\spaced{and}\!\! \Lambda_\star \!\in \argmax_\Lambda \mcM_{\theta_\star}(\Lambda).}
\end{equation*} 
Our adaption performs evidence maximisation with an affine model $h$ where the basis expansion $J$ is fixed.
Unlike \citet{Mackay1992Thesis},  we do not retrain the neural network.

\begin{figure}[t]
\vspace{-0.0cm}
    \centering
    \includegraphics[width=\linewidth]{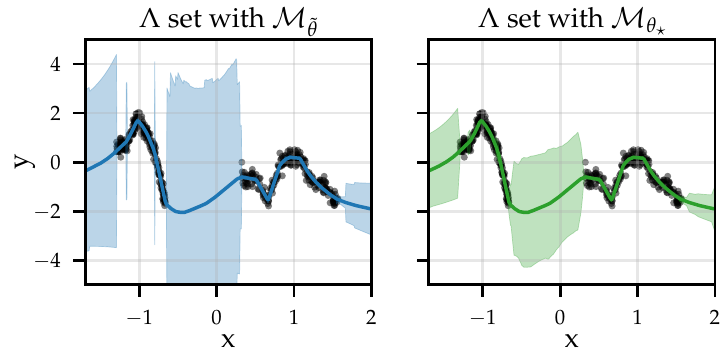}
    \vspace{-0.9cm}
    \caption{Linearised Laplace predictive mean and std-dev for a 2.6k parameter MLP trained on toy data from \citet{antoran2020depth}. Choosing $\Lambda$ with $\mcM_{\tilde\theta}$ yields error bars larger than the marginal std-dev of the targets. \Cref{rec:linear-weights} (using $\mcM_{\theta_{\star}}$) solves~this.}
    \label{fig:lin_vs_nn}
    \vspace{-0.25cm}
\end{figure}

In practice, we make one further approximation: rather than evaluating the curvature $\partial^2_\theta L(h(\theta, \cdot))$ afresh at 
successive modes of $\mcL_{h, \Lambda}$ found during the iterative procedure for computing $(\theta_\star, \Lambda_\star)$, we use the curvature at the linearisation point $H=\partial^2_\theta [L(h(\theta, \cdot))](\tilde\theta)$ throughout. This avoids the  expensive re-computation of the Hessian; experimentally we find that this does not affect the results (\cref{subsec:exp_assumptions}). The resulting model evidence expression matches that in \cref{eq:model-evidence}, with only the weights featuring in the norm~changed,
\begin{equation}\label{eq:model-evidence-new}
    \vspace{-0.05cm}
    \mcM_{\theta_{\star}}(\Lambda) = -\frac{1}{2}\left[ \|\theta_{\star}\|^{2}_{\Lambda}  + \log \det(\Lambda^{-1}H + I) \right] + C.
\end{equation}

\begin{recommendation}\label{rec:linear-weights}
While using the linearisation point $\tilde\theta$ in the construction of the feature expansion $J$ and the Hessian $H$ (as introduced in \cref{sec:preliminaries}), find a joint optimum $(\theta_\star, \Lambda_\star)$ for the feature-linear model and employ these to construct the corresponding model evidence $\mcM_{\theta_\star}$ (equation \cref{eq:model-evidence-new}) and to compute the predictive variance (equation \cref{eq:h-posterior}).
\end{recommendation}

We thus recommend employing a posterior distribution for $h$ of the same form as given in \cref{eq:h-posterior}, but with prior precision $\Lambda_\star$. 
In \cref{sec:experiments}, we provide extensive evidence supporting our alternative adaptation of the linearised Laplace model evidence. \Cref{fig:lin_vs_nn} provides an illustrative example.

\subsection{Linearised Laplace with normalised networks}\label{sec:normalised-networks}

We now study linearised Laplace in the presence of normalisation layers. For this, we put forth the following formalism:

\begin{definition}[Normalised networks] \label{def:normalised_networks}
We say that a set of networks $\mcF \subset \mcY^{\Theta\times \mcX}$ is normalised if $\Theta$ can be written as a \emph{direct sum} $\Theta' \oplus \Theta''$, with $\Theta''$ non-empty, such that for all networks $f \in \mcF$ and parameters $\theta' + \theta'' \in \Theta' \oplus \Theta''$,
\begin{equation*}
   f(\theta' + \theta'', \,\cdot\,) = f(\theta' + k\theta'', \,\cdot\,)
\end{equation*}
for all $k>0$.
\end{definition}
Throughout, we write $\theta',\theta''$ to denote the respective projections onto $\Theta',\Theta''$ of a parameter $\theta \in \Theta$; for ease of notation, we will assume these projections are aligned with a standard basis on $\Theta$. That is, we write our parameter vectors as the sum of $k \theta''$, which is only non-zero for normalised weights and $f$ is invariant to $k$, and $\theta'$, for which the opposite holds.  We illustrate this with an example in \cref{appendix:definition_example}.

Our formalism requires only that a single group of normalised parameters $\Theta''$ exists. However, by applying the definition repeatedly, we encompass networks with any number of normalisation layers, and all our results extend to this case. The formalism can be used to model the effects of layer norm \citep{Ba2016layer}, group norm \citep{He2020group} or batch norm \citep{ioffe2015batch}, and even some so-called normalisation-free methods \citep{Brock21propagation,Brock2021free}.

Our focus on normalised networks is motivated by the following observation:
\begin{proposition}\label{proposition:no-normalised-minimum}
For any normalised network $f$ and positive definite matrix $\Lambda$, the loss $\mcL_{f,\Lambda}$ has no local minima. 
\end{proposition}

To see this, note that the data term fit $L(f(\theta' + k\theta'', \cdot))$ is invariant to the choice of $k>0$, but we can always decrease the prior term $\|\theta'+ k\theta''\|^2_\Lambda$ by decreasing $k$. Since $k \in \R_{+}$ has no minimal value,  $\mcL_{f,\Lambda}(\theta'+ k\theta'') = L(f(\theta'+k\theta'', \cdot)) + \|\theta'+k\theta''\|^2_\Lambda$ has no local minima. 

As in \cref{sec:adapt-laplace}, minimisers of the linear loss $\mcL_{h, \Lambda}$ remain well-defined (the loss remains strictly convex). However, in this case, $\tilde\theta$ cannot minimise $\mcL_{h,\Lambda}$: the linearisation point minimises $\mcL_{h, \Lambda}$ only if it minimises $\mcL_{f, \Lambda}$ (\cref{proposition:f-not-then-h-not}). To correct this, from hereon we follow \cref{rec:linear-weights}.

An even larger concern raised by \cref{proposition:no-normalised-minimum} is that the linearisation point is identified only up to the scaling $k$ of the normalised parameters $\theta''$. Since $k$ is arbitrary, and does not affect the predictions of the neural network (by definition), it ought not affect the predictive variance returned by the linearised Laplace method. However, in general, it does. See \cref{fig:k_scan} for a demonstration of this. 
Nonetheless, as shown by the following proposition, it suffices to regularise the normalised parameters $\theta''$ separately from $\theta'$ to recover a unique predictive posterior independent of $k$.
\begin{figure}[H]
\vspace{-0.1cm}
    \centering
    \includegraphics[width=\linewidth]{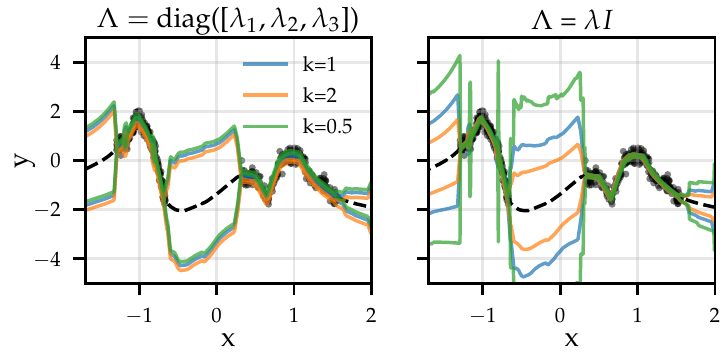}
    \vspace{-0.9cm}
    \caption{For a normalised MLP with an isotropic prior precision, modifying the scale of the normalised weights $\tilde\theta''$ in the linearisation point changes the error bars after hyper-parameter optimisation (right). Incorporating \cref{rec:factorised_prior} fixes the issue (left).}
    \label{fig:k_scan}
    \vspace{-0.15cm}
\end{figure}

\begin{proposition}\label{proposition:unique-posterior}
For normalised neural networks, using a regulariser of the form $\|\theta'\|^2_{\Lambda'} + \|\theta''\|^2_{\Lambda''}$ with $\Lambda'$ and $\Lambda''$ parametrised independently and chosen according to \cref{rec:linear-weights}, the predictive posterior $h(\theta,\cdot)$, $\theta \sim Q$ induced by a linearisation point $\theta'+k\theta''$ is independent of the choice of $k>0$.
\end{proposition}


Briefly, the result follows because the Jacobian entries corresponding to weights $k \theta''$ scale with $k^{-1}$.
This is illustrated in \cref{fig:contour} (leftmost plot), where
as we move further from the origin, weight settings of equal likelihood $L(f(\theta, \cdot))$ move further from each other. Given an un-scaled reference solution $(\theta_\star, \Lambda_\star)$,
as we vary $k$ in $k \theta''$, the linear model weights and prior precisions 
that simultaneously optimise $\mcL_{h, \Lambda_{\star}}$ and~$\mcM_{\theta_{\star}}$
 scale as $(\theta_\star', k\theta_\star'')$, and $(\Lambda_{\star}', k^{-2} \Lambda''_\star)$ respectively. These scalings cancel each other in the predictive posterior, which remains invariant.
When $\Lambda''$ can not change independently of $\Lambda'$, this cancellation does not occur.

By induction, \cref{proposition:unique-posterior} applies to networks with multiple normalisation layers. Our recommendation is thus:
\begin{recommendation}\label{rec:factorised_prior}
When using the linearised Laplace method with a normalised network, use an independent regulariser for each normalised parameter group present.
\end{recommendation}
An example of a suitable regulariser for a network with normalised parameter groups $\theta^{(1)}, \theta^{(2)}, \dotsc, \theta^{(L)}$ and non-normalised parameters $\theta'$ would be
\begin{equation*}
    \lambda'\|\theta'\|^2 + \lambda_1 \| \theta^{(1)}\|^2 + \lambda_2 \| \theta^{(2)}\|^2 + \dotsc + \lambda_L \| \theta^{(L)}\|^2
\end{equation*}
for independent parameters $\lambda', \lambda_1, \lambda_2, \dotsc, \lambda_L > 0$. In general, this involves setting independent priors for each layer of the network. We note that \cref{proposition:unique-posterior} holds even when $\Lambda_{\star}$ is found by evaluating the Hessian at the linearisation point $\tilde\theta$ instead of $\theta_{\star}$, as suggested in \cref{sec:adapt-laplace}.

\begin{figure*}[tbp]
\vspace{-0.05cm}
    \centering
    \includegraphics[width=\linewidth]{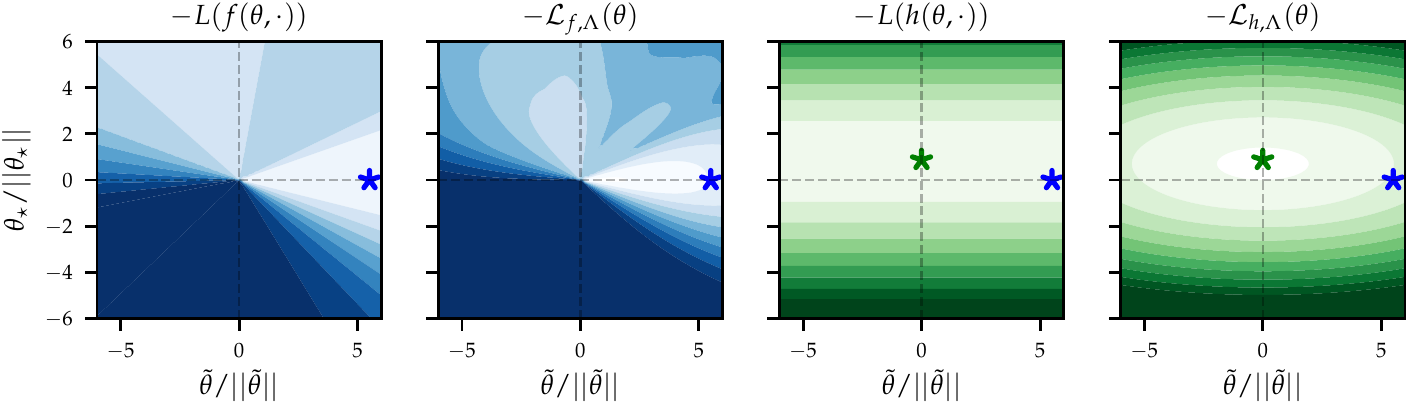}
    \vspace{-0.7cm}
    \caption{Log likelihood $L(f(\theta, \cdot))$ (left) and log posterior $\mcL_{f}$ (left middle) density for an MLP with layer norm, both plotted as functions of a 2d slice of the input layer weights.
    The horizontal axis corresponds to the direction of $\tilde\theta$ while the vertical to $\theta_{\star}$. The linearization point found with SGD $\tilde\theta$ (\textcolor{blue}{$\bigstar$})  is not an optima of $\mcL_{f, \Lambda}$ or $\mcL_{h, \Lambda}$. The linear model log posterior $\mcL_{h, \Lambda}$ (right) is convex and optimised~by~$\theta_{\star}$~(\textcolor{ForestGreen}{$\bigstar$}).
    }
    \vspace{-0.15cm}
    \label{fig:contour}
\end{figure*}

\subsection{Networks with a dense final layer}\label{sec:dense-final-layer}

We now look at networks with a dense linear final layer, a (very general) special case which provides further insight into linearised Laplace. These are networks of the form
\begin{equation}
\vspace{-0.025cm}
    f(\theta, \cdot) =   \phi(\theta_{d:}, \cdot) \cdot \theta_{:d},
\vspace{-0.025cm}
\end{equation}
where $\phi(\theta_{d:}, \cdot)$ is the output of the penultimate layer and $\theta_{:d}$ are the final layer weights.\footnote{The subindex ${d\!\!:}$ in $\theta_{d:}$ excludes last layer weights from $\theta$.} The derivative of the neural network with respect to the dense final layer weights is
\begin{equation*}
\vspace{-0.025cm}
    \partial_{\theta_{:d}} f(\theta, \cdot) = \phi(\theta_{d:}, \cdot),
\vspace{-0.025cm}
\end{equation*}
and thus $\phi(\theta_{d:}, \cdot)$ is contained within the Jacobian matrix. 
Consequently, the neural network output $f(\theta, \cdot)$ is always contained in the linear span of the Jacobian basis.
This motivates \cref{rec:linear-weights}, where we argue for the use of $\tilde\theta$ for network linearisation, as it allows for an easy linear model error-bars interpretation for the resulting uncertainty. 

Also, the form of the linearised model $h$ simplifies in the dense final layer case when the network is fully normalised. That is $f(\theta' + \theta'', \cdot) =   \phi(\theta_{d:}'', \cdot) \cdot \theta_{:d}'$. Here, the derivative of $\phi$ in the direction of $\tilde \theta''_{d:}$ is zero (see \cref{appendix:proof_directional_derivative}):
\begin{equation}\label{eq:directional-deriv-zero}
    \partial_{\theta''_{d:}} \phi(\tilde \theta''_{d:}, \cdot) \cdot \tilde \theta''_{d:} = 0.
\end{equation}
Thus cancellation occurs in \cref{eq:h-def}, yielding a model of the~form
\begin{equation}\label{eq:simple_linear_model}
    h(\theta, \cdot) =   \partial_\theta f(\tilde\theta, \cdot) \cdot \theta, \quad \theta \sim Q.
\end{equation}
That is, a linear model based on the features $J(\cdot) {=}\,\partial_\theta f(\tilde\theta, \cdot)$.

\subsection{Further implications and discussion}

Here, and in \cref{appendix:discussion}, we discuss details and implications of the presented recommendations and results.

\textbf{Finding the optimum of the linear model }
Our adapted linearised Laplace method requires identifying the joint stationary point $(\theta_{\star}, \Lambda_{\star})$. In general, this does not admit a closed-form solution. Instead, we alternate gradient-based optimisation of $\mcL_{h, \Lambda}$ and $\mcM_{\theta}$. In \cref{app:algorithm}, we provide a derivation for the gradient of $\mcL_{h, \Lambda}$, \cref{alg:linear_MAP_algorithm} for its computation and discuss implementation trade-offs. 
For normalised networks with dense output layers, implementing the simplified linear model \cref{eq:simple_linear_model} directly yields faster and more stable optimisation. 
Obtaining the gradients of $\mcM_{\theta}$ involves computing Hessian log-determinants, which in turn requires approximations in the context of large networks. 

\textbf{Magnitude of linearisation point in normalised networks }
Optimising a normalised neural network returns a solution for the normalised weights (those in $\Theta''$) up to some scaling factor $k>0$. How is $k$ determined?
Recall, from \cref{eq:directional-deriv-zero}, that for any $\theta'' \in \Theta''$, the directional derivative of the NN output in the direction of $\theta''$ is zero. This is also illustrated in \cref{fig:contour}. With this in mind, the dynamics of optimisation can be understood by analogy to a Newtonian system in polar coordinates. The weights are a mass upon which the data fit gradient acts as a tangential force. When discretised, this gradient pushes the weights away from zero.
On the other hand, regularisation from the prior term acts like a centripetal force, pushing the weights towards the origin. 
The resulting $k$ is thus proportional to the variance of the gradients of $\theta''$, and as such dependent on the learning rate and batch size hyperparameters, while being inversely proportional to the regularisation strength, e.g. weight decay. This has been studied extensively in the optimisation literature, including \citet{Laarhoven2017L2,Hoffer18norm,Cai2019quantitative,li2020intrinsic,Lobacheva2021periodic}. 

\textbf{On network biases in the Jacobian feature expansion }
Most normalisation techniques introduce scale invariance by dividing subsets of network activations by an empirical estimate of their standard deviation. These activations depend on the values of both weights and biases.
On the other hand, practical use of linearised Laplace commonly considers uncertainty due to only network weights \citep{daxberger2021bayesian,Maddox2021Adaptation}, excluding bias entries from Jacobian and Hessian matrices.
This departure from our assumptions can break the scale invariance necessary for \cref{lemma:scaling}. 
Whether invariance is preserved for the weights in the bias-exclusion setting depends on the relative effect of weights and biases on each subset of normalised activations. 
Invariance is preserved if the biases have small impact.
Empirically, we find that the inclusion (or exclusion) of biases does not alter the improvements obtained from applying our recommendations (see \cref{fig:hessian_bias_choice}). 

\textbf{Implications for the (non-linearised) Laplace method }
The (non-linearised) Laplace method \citep{ritter2018scalable,Kristiadi2020being} approximates the intractable posterior $P$ by means of a quadratic expansion around an optima, but without the linearisation step given in equation \cref{eq:h-def}. As discussed in \cref{sec:adapt-laplace}, when employing stochastic optimisation, early stopping, or normalisation layers, we will not find a minimiser of $\mcL_{f,\Lambda}$. Without a well-behaved surrogate linear model loss to fall back on, the Laplace method can yield very biased estimates of the model evidence.

\section{Formal results and proofs}\label{sec:proofs}

\begin{proposition}\label{proposition:f-not-then-h-not} For network $f$ with linearisation $h$ about $\tilde\theta$ and a positive definite regulariser $\Lambda$, if $\tilde\theta$ is not a stationary point of $\mcL_{f, \Lambda}$, it is not a local minimum of $\mcL_{h,\Lambda}$.
\end{proposition}
\begin{proof}
Since $\tilde\theta$ is not a stationary point of $\mcL_{f, \Lambda}$, the gradient $\partial_\theta \mcL_{f,\Lambda}(\tilde\theta)$ is not identically zero. But $\partial_\theta \mcL_{f,\Lambda}(\tilde\theta)$ equal to
\begin{align*}
    \sum_i \partial_f &\ell(f(\tilde\theta, x_i), y_i) \cdot \partial_\theta f(\tilde\theta, x_i) + \partial_\theta [\|\theta\|^2_{\Lambda}](\tilde\theta)\\
    &= \sum_i \partial_h \ell(h(\tilde\theta, x_i), y_i) \cdot \partial_\theta h(\tilde\theta, x_i) + \partial_\theta [\|\theta\|^2_{\Lambda}](\tilde\theta),
\end{align*}
which is in turn equal to $\partial_\theta \mcL_{h,\Lambda}(\tilde\theta)$. Since this is thus non-zero, $\tilde\theta$ cannot be a local minimum of $\mcL_{h,\Lambda}$.
\end{proof}

The proof of \cref{proposition:no-normalised-minimum} was sketched in text; we omit it.

For the proof of \cref{proposition:unique-posterior}, we use the following notation. Consider a linearisation point $\theta' + \theta''$, with corresponding linearised function $h$, basis function $J$ and Hessian $H$. For $k>0$, $h_k,J_k,H_k$ denote these quantities corresponding to a linearisation point $\theta_{k} \coloneqq \theta' + k\theta''$. Moreover, we write
\begin{equation*}
    J = \begin{bmatrix} J' \\ J''
    \end{bmatrix} 
    \spaced{and} 
    H = \begin{bmatrix}
        H' & X^T \\
        X & H''
    \end{bmatrix}
\end{equation*}
for the sub-entries of $J$ and $H$ with dependencies on $\theta'$ and $\theta''$ respectively, with $X$ containing cross-terms. We refer to sub-entries of $J_k$ and $H_k$ in the same manner. 

With notation in place, we have the following scaling result:

\begin{lemma}\label{lemma:scaling}
Let $f$ be a normalised network and consider two alternative linearisation points $\tilde\theta' + \tilde\theta''$ and $\tilde\theta' + k \tilde\theta''$ for some $k>0$. Then,
\begin{equation*}
    \begin{bmatrix} J'_k \\ k^{-1} J''_k
    \end{bmatrix} = J
    \spaced{and}
    \begin{bmatrix}
        H'_k & k^{-1}X_k^T \\
        k^{-1}X_k & k^{-2}H''_k
    \end{bmatrix} = H.
\end{equation*}
Moreover, for all $\theta \in \Theta$, $h_k(\theta' + k\theta'', \cdot) = h(\theta' + \theta'', \cdot)$.
\end{lemma}

The proof of the above lemma is presented in \cref{appendix:scaling-proof}. We now use it to show how the optimal weights and regularisation parameters scale with the parameter $k$.

\begin{lemma}\label{lemma:stationary-points}
For $k > 0$, let $h_k$ be a linearisation of a normalised network $f$ about $\tilde\theta'+k\tilde\theta''$. Then $(\theta_k, \Lambda_k)$ are an optima of the resulting objectives $(\mcL_{h_k,\Lambda_k}, \mcM_{\theta_k})$ respectively if and only if they are of the form 
\begin{equation*}
    (\theta_k, \Lambda_k) = (\theta'_\star + k\theta''_\star, \Lambda'_\star + k^{-2}\Lambda''_\star)
\end{equation*}
where $(\theta_\star, \Lambda_\star)$ are optima of $(\mcL_{h,\Lambda_\star},\mcM_{\theta_\star})$ with $h$ a linearisation of $f$ about $\tilde\theta' + \tilde\theta''$.
\end{lemma}

\begin{proof} To prove the result, we will show that $\mcL_{h_k, \Lambda_k}(\theta'+k\theta'') = \mcL_{h,\Lambda_\star}(\theta' + \theta'')$ for all $\theta' + \theta'' \in \Theta$ and $\mcM_{\theta_k}(\Lambda' + k^{-2}\Lambda'') = \mcM_{\theta_\star}(\Lambda' + \Lambda'')$ for all strictly diagonal positive matrices $\Lambda',\Lambda''$ of compatible sizes. Then, the result follows by noting that for $k>0$ fixed, the mappings $\theta' + \theta'' \mapsto \theta' + k\theta''$ and $\Lambda' + \Lambda'' \mapsto \Lambda' + k^{-2}\Lambda''$ are bijections.

Consider the objective $\mcL_{h_k,\Lambda_k}$. By definition, $\mcL_{h_k,\Lambda_k}(\theta' + k\theta'')$ is given by
\begin{align*}
    L(h_k(\theta' + &k\theta'', \cdot)) + \|\theta'\|^2_{\Lambda'_k} + \|k\theta''\|^2_{\Lambda''_k} \\
    &= L(h(\theta' + \theta'', \cdot)) + \|\theta'\|^2_{\Lambda'_\star} + \|\theta''\|^2_{\Lambda''_\star},
\end{align*}
where the equality follows by \cref{lemma:scaling} and the definition of $\Lambda_k$. The right hand side is equal to $\mcL_{h,\Lambda_\star}(\theta' + \theta'')$ proving the first equality.

Consider the objective $\mcM_{\theta_k}$. For our claim, we need to show that
\begin{align*}
    \|\theta_k\|_{\Lambda' + k^{-2}\Lambda''} &+ \log \frac{\det(H_k + \Lambda' + k^{-2}\Lambda'')}{\det(\Lambda' + k^{-2}\Lambda'')} \\ 
    &= \|\theta_\star\|_{\Lambda' + \Lambda''} + \log \frac{\det(H + \Lambda' + \Lambda'')}{\det(\Lambda' + \Lambda'')}.
\end{align*}
The equality $\|\theta_k\|_{\Lambda' + k^{-2}\Lambda''} = \|\theta_\star\|_{\Lambda' + \Lambda''}$ holds trivially. We now show equality of the determinants. Let $d,D$ denote the dimensions of $\Theta'$ and $\Theta''$ respectively. By the Schur determinant lemma, the numerator $\det(H_k + \Lambda' + k^{-2}\Lambda'')$ is equal to
\begin{align*}
    \det(H''_k + \frac{\Lambda''_{d:}}{k^2})\,\det(H'_k + \Lambda'_{:d} - X(H''_k + \frac{\Lambda''_{d:}}{k^2})^{-1}X^T),
\end{align*}
where $\Lambda'_{:d} = [\Lambda'_{ij} \colon i,j \leq d]$ and $\Lambda''_{d:}$ is defined similarly. Using \cref{lemma:scaling}, $\det(H''_k + \frac{\Lambda''_{d:}}{k^2}) = (\frac{1}{k^2})^{D} \det(H'' + \Lambda''_{d:})$. Expanding the Schur complement term and using \cref{lemma:scaling} shows that it is independent of $k$. In turn, the denominator is given by
\begin{align*}
    \det(\Lambda' + k^{-2}\Lambda'') &= (\frac{1}{k^2})^{D}\det(\Lambda'_{:d})\,\det(\Lambda''_{d:}) \\
    &= (\frac{1}{k^2})^{D} \det(\Lambda' + \Lambda''),
\end{align*}
The $(\frac{1}{k^2})^{D}$ terms in the numerator and denominator cancel, yielding the claim.
\end{proof}

\begin{proof}[Proof of \cref{proposition:unique-posterior}] 
Using \cref{lemma:scaling} and \cref{lemma:stationary-points} and the notation defined therein,
\begin{equation*}
    \|J\|^2_{(H + \Lambda_\star)^{-1}} = \|J_k\|^2_{(H_k + \Lambda_k)^{-1}}.
\end{equation*}
Thus the errorbars induced by linearising about $\tilde\theta'+\tilde\theta''$ and $\tilde\theta'+k\tilde\theta''$ are equal for all $k > 0$.
\end{proof}
Note that the proof of the results required for \cref{proposition:unique-posterior} depends crucially on being able to scale $\Lambda''_k$ with $k$ while keeping $\Lambda'_k$ fixed. This motivates \cref{rec:factorised_prior}.

\section{Experiments}\label{sec:experiments}

We proceed to provide empirical evidence for our assumptions and recommendations. 
Specifically, in \cref{subsec:exp_assumptions}, we validate the assumptions made in \cref{sec:main_results}.
Then, in \cref{subsec:architectures}, we demonstrate that our recommendations yield improvements across a wide range of architectures.
In these first two subsections, we employ networks containing at most $46$k weights, since this is the largest model for which we can tractably compute the Hessian on an A100 GPU. This choice avoids confounding the effects described in \cref{sec:main_results} with any further approximations.
In \cref{subsec:large_models}, we show that our recommendations yield performance improvements on the $23$M parameter ResNet-50 network while employing the standard KFAC approximation to the Hessian \citep{Martens15kron,daxberger2021laplace}.

Unless specified otherwise, we:
1. train a NN to find $\tilde\theta$ using standard stochastic optimisation algorithms. 2. linearise the network about $\tilde\theta$ as in \cref{eq:h-def}. 3. optimise the linear model weights using $\mcL_{h, \Lambda}$ (\cref{alg:linear_MAP_algorithm}) and layer-wise regularisation parameters with $\mathcal{M}_{\theta_{\star}}$ \cref{eq:model-evidence-new}. 4. compute the posterior predictive distribution as in \cref{eq:h-posterior}.
We repeat this procedure with 5 random seeds and report mean results and standard error.
For each seed, the methods compared produce the same mean predictions $f(\tilde\theta, \cdot)$, only differing in their predictive variance. In this setting, the test Negative Log-Likelihood (NLL, lower is better) can be understood as a measure of uncertainty miscalibration.
The full details of the setup used for each experiment are provided in \cref{app:experimental_setup}.

\subsection{Validation of modelling assumptions} \label{subsec:exp_assumptions}
We validate the key conjectures stated in \cref{sec:main_results}.
If not specified otherwise, we employ a 46k parameter ResNet \citep{he2016deep} with batch-normalisation after every convolutional layer. The output layer is dense, satisfying \cref{eq:simple_linear_model}. 

\textbf{Choice of Hessian } In \cref{sec:adapt-laplace}, we suggest evaluating the Hessian of $\mcL_{h}$ at the linearisation point $\tilde\theta$ (instead of $\theta_{\star}$) for model evidence optimisation \cref{eq:model-evidence-new}. This avoids the need to recompute the Hessian throughout optimisation. \Cref{fig:hessian_bias_choice} (left) shows how the improvement from using the recommended model evidence $\mcM_{\theta_{\star}}$, as opposed to $\mcM_{\tilde\theta}$, dominates the effect of the choice of Hessian evaluation~point.

\begin{figure}[htbp]
\vspace{-0.075cm}
    \centering
    \includegraphics[width=\linewidth]{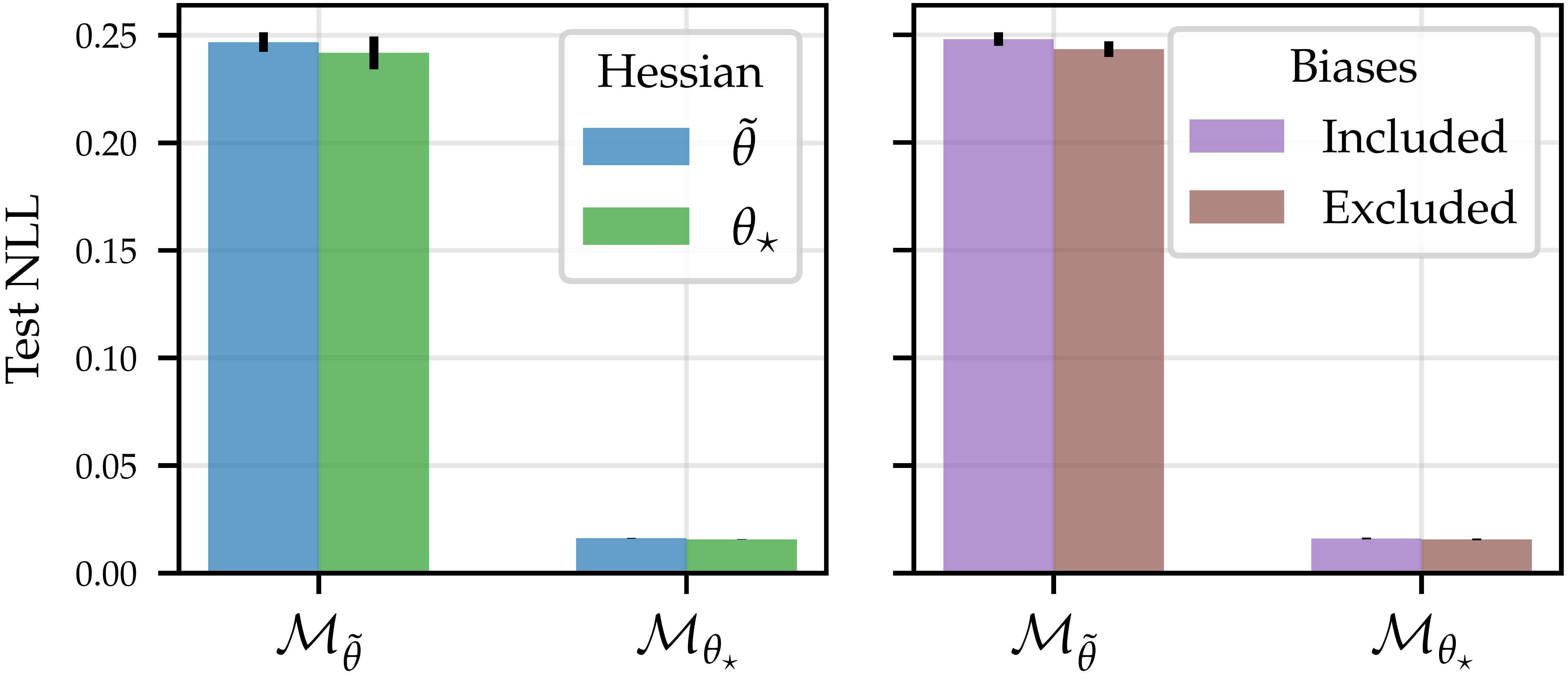}
    \vspace{-0.75cm}
    \caption{Comparison of the test NLL improvement obtained when switching from $\mcM_{\tilde\theta}$ to $\mcM_{\theta_{\star}}$ to optimise the prior precision $\Lambda$ relative to the impact of (\emph{left}) evaluating the Hessian at $\tilde\theta$ or $\theta_{\star}$, and (\emph{right}) excluding network biases from the basis functions. Both plots use a 46k ResNet with batch norm trained on MNIST.}
    \label{fig:hessian_bias_choice}
    \vspace{-0.175cm}
\end{figure}


\begin{table*}[th]
\vspace{-0.10in}
\centering
\caption{Validation of recommendations across architectures. All results are reported as negative log-likelihoods (lower is better). In each column, the best performing method is bolded. For each $\mathcal{M}$,  if single or layerwise $\lambda$ optimisation performs better, it is underlined.}
    \label{tab:model_sweep}
\vspace{0.03in}
\begin{tabular}{llllllll}
\toprule
                 &                     &                    \textsc{Transformer} &                            \textsc{CNN} &                         \textsc{ResNet} &                     \textsc{Pre-ResNet} &                          \textsc{FixUp} & \textsc{U-Net} \\
\midrule
\multirow{2}{*}{$\mcM_{\theta_{\star}}$} & single $\lambda$ &  \textbf{0.162} {\tiny \color{gray} $\pm$ 0.042} &  \textbf{0.025} {\tiny \color{gray} $\pm$ 0.000} &  \textbf{0.017} {\tiny \color{gray} $\pm$ 0.000} &  0.017 {\tiny \color{gray} $\pm$ 0.000} &  \textbf{0.055} {\tiny \color{gray} $\pm$ 0.006} & -1.793 {\tiny \color{gray} $\pm$ 0.050}  \\
                 & layerwise $\lambda$ &  \textbf{0.162} {\tiny \color{gray} $\pm$ 0.042} &  \textbf{0.025} {\tiny \color{gray} $\pm$ 0.000} &  \textbf{0.016} {\tiny \color{gray} $\pm$ 0.001} &  \underline{\textbf{0.016}} {\tiny \color{gray} $\pm$ 0.000} &  0.061 {\tiny \color{gray} $\pm$ 0.005} & \underline{\textbf{-2.240}} {\tiny \color{gray} $\pm$ 0.027}\\
\cmidrule{1-8}
\multirow{2}{*}{$\mcM_{\tilde\theta}$} & single $\lambda$ &  0.310 {\tiny \color{gray} $\pm$ 0.060} &  0.253 {\tiny \color{gray} $\pm$ 0.001} &  0.252 {\tiny \color{gray} $\pm$ 0.006} &  \underline{0.220} {\tiny \color{gray} $\pm$ 0.004} &  \underline{0.153} {\tiny \color{gray} $\pm$ 0.021} & -1.164 {\tiny \color{gray} $\pm$ 0.052}\\
                 & layerwise $\lambda$ &  \underline{\textbf{0.162}} {\tiny \color{gray} $\pm$ 0.042} &  \underline{0.205} {\tiny \color{gray} $\pm$ 0.002} &  \underline{0.236} {\tiny \color{gray} $\pm$ 0.005} &  0.239 {\tiny \color{gray} $\pm$ 0.004} &  0.200 {\tiny \color{gray} $\pm$ 0.018} & \underline{-1.703} {\tiny \color{gray} $\pm$ 0.023} \\
\bottomrule
\end{tabular}
\vspace{-0.1in}
\end{table*}

\textbf{Dependence on $k$ for isotropic precisions } In \cref{fig:k_scan}, we illustrate the dependence of the predictive posterior on the scale of normalised weights $k$ for an isotropic prior precision $\Lambda {=}\,\lambda I$, i.e. \cref{rec:factorised_prior} is ignored. We use a 2.6k parameter 2 hidden layer fully connected NN with layer norm after every layer except the last and a 1d regression task.
Changing $k$ changes the optimal $\lambda$ and, consequently, the predictive uncertainty changes.
With layer-wise $\lambda$ this effect vanishes (as predicted by \cref{proposition:unique-posterior}).


\textbf{Treatment of NN biases }
 Excluding the Jacobians of network biases from our basis function expansion breaks the scaling properties presented in \cref{lemma:scaling}. In \cref{fig:hessian_bias_choice} (right), we show that the effect of excluding biases is dominated by the choice of the model evidence between $\mcM_{\theta_{\star}}$ and $\mcM_{\tilde\theta}$.
 
\textbf{Early stopping }
We evaluate whether more thorough optimisation of the NN weights with $\mcL_{f}$ leads to a linearisation point $\tilde\theta$ closer to $\theta_{\star}$ in the sense of the implied optimal regularisation and induced posterior predictive distribution. We perform this analysis on normalised and unnormalised networks (which use the non scale-invariant FixUp regularisation instead \citep{Zhang19FixUp}),
since $\tilde\theta$ is guaranteed to never match $\theta_{\star}$ for the former.
Surprisingly, the Wasserstein-2 distance between predictive distributions obtained with $\mcM_{\tilde \theta}$ and $\mcM_{\theta_{\star}}$ increases with more optimisation steps in both cases. Thus, more thorough optimisation does not~help.

\begin{figure}[htbp]
\vspace{0.15cm}
    \centering
    \includegraphics{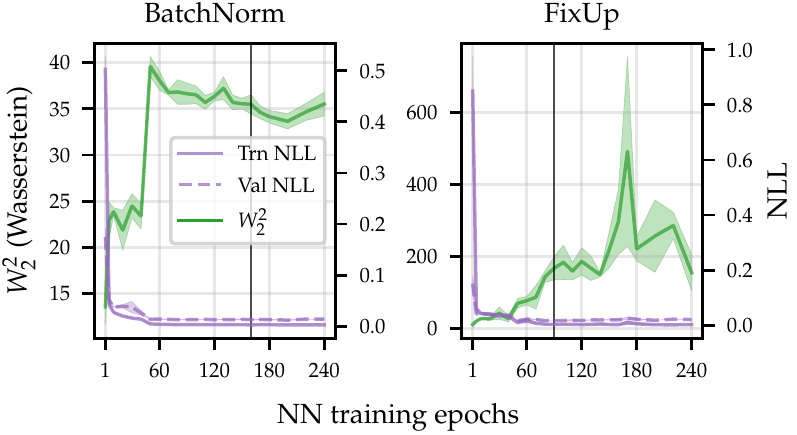}
    \vspace{-0.2cm}
    \caption{Wasserstein distance between predictive posteriors obtained when using $\mcM_{\tilde\theta}$ and $\mcM_{\theta_{\star}}$ at multiple NN training stages. The vertical black line indicates optimal (val-based) early stopping.}
    \label{fig:early_stopping}
    \vspace{-0.35cm}
\end{figure}




\subsection{Validating recommendations across architectures}\label{subsec:architectures}

We evaluate the utility of \cref{rec:linear-weights}, and \cref{rec:factorised_prior} on a range of architectures and tasks:  1. a transformer architecture on the pointcloud-MNIST variable length sequence classification task. This model uses layer norm in alternating layers. 2. a LeNet-style CNN with batch norm placed after every convolutional layer and a dense output layer tasked with MNIST classification. 3. a ResNet with batch norm after every layer except the dense output layer (MNIST classification). 4. the same ResNet but with batch norm substituted by (non scale-invariant) FixUp regularisation (MNIST classification). 5. a pre-ResNet \citep{preresnet}. This architecture differs from ResNet in that batch norm is placed before each weight layer instead of after them; the implication is that there is only 1 normalised group of weights encompassing all weights but those of the dense output layer (MNIST classification). 6. a fully convolutional U-net autoencoder tasked with tomographic reconstruction (regression) of a KMNIST character from a noisy low-dimensional observation. We reproduce the setting of \citet{barbano2021dip}. Group norm is placed after every layer except the last, which is convolutional. 

As shown in \cref{tab:model_sweep}, the application of \cref{rec:linear-weights} yields notably improved performance across all settings. Applying \cref{rec:factorised_prior} yields modest improvements for classification networks with normalisation layers but large improvements for the U-net. Interestingly, layer-wise regularisation degrades performance in the FixUp ResNet.

\begin{table}[th]
\vspace{-0.2cm}
\centering
\caption{Test negative log-likelihoods for ResNet-50 on CIFAR10.}
\label{tab:resnet50_cifar}
\resizebox{\columnwidth}{!}{%
\begin{tabular}{llll}
\toprule
 &  & \textsc{batch norm} & \textsc{FixUp} \\ \midrule
\multirow{2}{*}{$\mcM_{\theta_{\star, \text{simple}}}$} & single $\lambda$ & \textbf{0.112} {\tiny \color{gray} $\pm$ 0.004} & 0.128 {\tiny \color{gray} $\pm$ 0.000} \\
 & layerwise $\lambda$ & \textbf{0.109} {\tiny \color{gray} $\pm$ 0.003} & \underline{\textbf{0.096}} {\tiny \color{gray} $\pm$ 0.000} \\
\midrule       \multirow{2}{*}{$\mcM_{\theta_{\star}}$} & single $\lambda$ & 0.190 {\tiny \color{gray} $\pm$ 0.005} & 0.249 {\tiny \color{gray} $\pm$ 0.002} \\
 & layerwise $\lambda$ & 0.194 {\tiny \color{gray} $\pm$ 0.009} & \underline{0.193} {\tiny \color{gray} $\pm$ 0.001} \\
\midrule
\multirow{2}{*}{$\mcM_{\tilde\theta}$} & single $\lambda$ & 0.570 {\tiny \color{gray} $\pm$ 0.004} & 0.412 {\tiny \color{gray} $\pm$ 0.000} \\
 & layerwise $\lambda$ & 0.567 {\tiny \color{gray} $\pm$ 0.004} & \underline{0.360} {\tiny \color{gray} $\pm$ 0.000} \\ \bottomrule
\end{tabular}%
}
\vspace{-0.2cm}
\end{table}

\subsection{Large scale models}\label{subsec:large_models}

We validate our recommendations on the 23M parameter ResNet-50 network trained on the CIFAR10 dataset. This model places batch norm after every layer except the dense output layer. We also consider a normalisation-free FixUp ResNet-50. 
\Cref{tab:resnet50_cifar} shows that both of our recommendations yield better test NLL, with larger gains obtained by the batch norm network.
The normalisation-free FixUp setting does not simplify as in eq.~\cref{eq:simple_linear_model}. Nonetheless, assuming the simplified model evidence, denoted $\mcM_{\theta_{\star, \text{simple}}}$, when obtaining the linear model optima yields improved performance for all models. We expand on this in \cref{appendix:additional_results_arq}.

\section{Discussion}
We have identified and addressed two pitfalls of a na\"ive application of linearised Laplace to modern NNs. First, the optima of the loss function is not found in practice. This invalidates the assumption that the point at which we linearise our model is stationary. However, every linearisation point implies an associated basis function linear model. As we use this model to provide errobars, we propose to choose hyperparameters using the evidence of this model. This requires only the solving of a convex optimisation problem, one much simpler than NN optimisation. Second, normalisation layers introduce an invariance to the scale of NN weights and thus the linearisation point can only be identified up to a scaling factor. We show that to obtain a predictive posterior that is invariant to this scaling factor, the regulariser must be independently parametrised for each normalised group of weights, e.g. different layers. 
Our experiments confirm the effectiveness of these recommendations across a wide range of model architectures and sizes.

\clearpage
\section*{Acknowledgements}

The authors would like to thank Marine Schimel for helpful discussions. JA acknowledges support from Microsoft Research, through its PhD Scholarship Programme, and from the EPSRC.
RB acknowledges support from the i4health PhD studentship
(UK EPSRC EP/S021930/1), and from The Alan Turing Institute (UK EPSRC EP/N510129/1).
ED acknowledges funding from the EPSRC and Qualcomm. JUA acknowledges funding from the EPSRC, the Michael E. Fisher Studentship in Machine Learning, and the Qualcomm Innovation Fellowship. This work has been performed using resources provided by the Cambridge Tier-2 system operated by the University of Cambridge Research Computing Service (http://www.hpc.cam.ac.uk) funded by EPSRC Tier-2 capital grant EP/T022159/1.

\bibliography{references}
\bibliographystyle{icml2022}

\newpage
\appendix
\onecolumn

\section{Additional discussion of related work}

We commence by providing an overview of previous work on the Laplace approximation for neural networks while discussing its relation to the present paper. We then delve into alternative approximate Bayesian inference methods for neural networks.

\subsection{The linearised Laplace approximation for neural networks}

The seminal work on Bayesian inference for NNs of \citet{Mackay1992Thesis} approximated the intractable posterior over the neural network weight space using the Laplace approximation. 
The so-called ``Evidence Framework'' laid out by \citet{Mackay1992Thesis} consists of an EM-style iterative algorithm where a NN is fit and Laplace approximated (E-step), and then this posterior approximation is used to update hyperparameters (M-step) such that the Laplace marginal likelihood estimate is maximised. The author further employed a local linearisation of the NN to obtain a closed form approximation of the predictive posterior. However, \citet{Mackay1992Thesis} estimated the posterior's curvature using the full Hessian, while mentioning that the Generalised Gauss Newton matrix (GGN), which corresponds to the curvature of the surrogate linear model's loss, is a computationally cheaper alternative. Since then, the GGN approximation has become ubiquitous \citep{bishop2006pattern}.

The local linearisation employed by \citet{Mackay1992Thesis} for estimating the predictive posterior was seen as an approximation chosen for convenience. 
With this in mind, \citet{Lawrence_thesis} forgoed the linearisation step and used Monte Carlo sampling to marginalise the distribution over network weights. This resulted in very poor performance. \citet{ritter2018scalable} also employed sampling in their work. Similarly to \citet{Lawrence_thesis}, these authors found that the Laplace approximation can place large amounts of probability mass in low density regions of the posterior, yielding poor results. They circumvented this issues by damping their posterior covariance estimate, making it closer to a point mass at the mode.

With the proliferation of larger neural networks with higher dimensional weight spaces, computation of the Hessian became intractably expensive and the Laplace approximation fell out of favour relative to other approximate inference methods discussed in \cref{appendix:non_laplace_BNNs}. The work of \citet{ritter2018scalable} was the first to break this trend by proposing the use of the the Kronecker factorised (KFAC) approximation \citep{Martens15kron,botev2017Optimisation} to approximate the NN Generalised Gauss Newton (GGN) curvature matrix, making Laplace's method scalable to modern architectures.

Building upon the above work, \citet{Khan2019Approximate} made the observation that the GGN posterior precision approximation coincides with the exact posterior of the locally linearised neural network, when seen as a model in its own right. \citet{Khan2019Approximate} also suggested the use of the marginal likelihood of this surrogate linear model as a proxy for that of the neural network. Following this, \citet{foong2019inbetween,Immer2021Improving} made the observation that assuming linearisation at prediction time significantly improves performance. 
The latter work \citep{Immer2021Improving} expanded on this topic by discussing the view of the GGN Laplace approximation as a model switch into a generalised linear model. The authors note that the linearisation point may not be the optima of the NN loss in practise and that the Laplace distribution does not yield the exact posterior for the linear model when employing non-Gaussian likelihoods e.g. in the classification setting. Motivated by this, \citet{Immer2021Improving} proposed to ``refine'' the linear model's posterior, which they did using variational inference. This provided some improvements in terms of predictive performance. Our \cref{rec:linear-weights} can be interpreted as suggesting a form of posterior refinement where the posterior covariance is kept fixed. 
In \cref{sec:normalised-networks}, we show that this refinement applied jointly with prior precision refinement is in fact necessary to obtain a well-defined posterior predictive distribution when applying linearised Laplace to networks with normalisation layers.

\citet{Kristiadi2020being} showed that reliable uncertainty estimates can be obtained by only applying the Laplace approximation to the last layer's weights.
\citet{daxberger2021bayesian} avoided the need for KFAC approximations by only considering the curvature of a non-last-layer axis-aligned subspace of the weights. Although both of these methods are scalable to large models, they give a deterministic treatment to most network parameters, which eliminates the possibility of model selection.

Building upon the work of \citet{Khan2019Approximate}, \citet{Immer2021Marginal} demonstrated the application of the linearised Laplace model evidence for hyperparameter and architecture selection in neural networks. Their proposed method differs from the one studied in this paper however, in that \citep{Immer2021Marginal} apply an online approach; model hyperparameters are learnt simultaneously with network weights. The M step is performed every few epochs of NN training, using the unconverged neural network weights in the model evidence estimate. Although this objective is clearly biased, the use of a different Jacobian basis expansion at each M step brings the algorithm of \citet{Immer2021Marginal} outside of the scope of the study of the present work.
Be that as it may, the authors note that their method does not work well in the presence of normalisation layers. This motivates them to use FixUp regularisation \citep{Zhang19FixUp} instead. Our work finds that FixUp only partially ameliorates the degradation in model selection capabilities brought by using the norm of the linearisation point in the model selection objective. Also related is the work of \citet{maddox2020rethinking}, who built upon the analysis of \citet{Mackay1992Thesis} by showing that the Hessian eigenspectrum can be used to predict the generalisation properties of modern neural networks.

\subsection{Approximate inference and uncertainty estimation with neural networks}\label{appendix:non_laplace_BNNs}

As a more accurate alternative to the Laplace approximation, \citet{neal1995bayesian} introduced Hamiltonian Monte Carlo. Although asymptotically unbiased, this method requires a large number of full-batch gradient evaluations in order to draw independent samples. As a result, the cost of running Hamiltonian Monte Carlo in the modern large-network large-dataset setting is prohibitive.  It is worth mentioning the work of \citet{izmailov2021what} who leverage vast computational resources to demonstrate the application of Hamiltonian Monte Carlo in the modern setting. Unlike the Laplace approximation, sampling methods do not provide straight forward access to marginal likelihood estimates. \citet{Neal2001Annealed} introduced annealed importance sampling for this purpose, but its reliance on Hamiltonian Monte Carlo precludes its practical application to modern neural networks. Stochastic gradient sampling methods \citep{Welling2011SGLD,Chen2014SGHMC,wenzel2020good} relax the full batch requirement of the above methods but can be heavily biased as a result \citep{Betancourt2015Incompatibility}.

Another class of approximate inference methods for neural networks are those that employ variational inference \citep{hinton1993keeping}. These re-formulate inference of the Bayesian posterior as the optimisation of the parameters of a variational distribution \citep{graves2011practical,blundell2015weight,Osawa2019practical,Dusenberry2020rank}. The objective being optimised is known as the Evidence Lower BOund (ELBO) and, in principle, it can be used as a surrogate objective for model selection. However, for neural network models, looseness in the constructed ELBOs has historically prevented these objectives from being used for this purpose. Two recent notable exceptions are the work on tighter variational bounds for NNs of \citet{Tomczak2021Collapsed} and \citet{ober2021global}. However, their performance for model selection has yet to be evaluated extensively. 
Closely related to variational methods are those that use expectation propagation \citep{hernandez2015probabilistic} and alpha divergence minimisation \citep{Lobato2016black}.

Finally, we note that although uncertainty estimation is the main goal of Bayesian deep learning, the most successful methods for uncertainty estimation with deep networks have traditionally been based on ensembling \citep{lakshminarayanan2017, ashukha2020pitfalls,snoek2019can,Dusenberry2020rank}. The linearised Laplace method is unique in being one of the only Bayesian deep learning methods which performs competitively with ensembling methods \citep{daxberger2021bayesian}.

\nocite{Bhatt2021transparency}
\nocite{antoran2021clue}
\nocite{antoran2019disentangling}

\section{Example to illustrate \cref{def:normalised_networks}} \label{appendix:definition_example}

We now provide an example to illustrate how \cref{def:normalised_networks} constructs the parameter space as a direct sum of the subsapces of normalised and non-normalised weights.

Consider an MLP $f: \Theta \times \mathcal{X} \to \mathcal{Y}$ with a single input dimension $\mathcal{X} = \mathbb{R}$, single output dimension $\mathcal{Y} = \mathbb{R}$, single hidden layer and 2 hidden units. We apply layer norm after the input layer parameters. The model parameters are $\theta = [\theta_{1}, \theta_{2}, \theta_{3}, \theta_{4}] \in \Theta$ with $\theta_{1}, \theta_{2}$ belonging to the input layer and $\theta_{3}, \theta_{4}$ to the readout layer.  We assume there are no biases without loss of generality. 

Denoting the outputs of the first parameter layer $a = [\theta_{1}x, \theta_{2}x]$, layer norm applies the function 
\begin{gather*}
   \frac{ a - \mathbb{E}[a]}{\sqrt{\text{Var}(a)}} \gamma + \beta \spaced{with} \mathbb{E}[a] = 0.5 \theta_{1}x + 0.5 \theta_{2}x  \spaced{and} \text{Var}(a) = 0.5 (\theta_{1}x - \mathbb{E}[a])^2 + 0.5 (\theta_{2}x - \mathbb{E}[a])^2
\end{gather*}
for $\gamma \in \mathbb{R},\, \beta \in \mathbb{R}$. Now take $k \in \mathbb{R}_{+}$ to see that the output of the layernorm layer is invariant to scaling the input layer parameters by $k$
\begin{gather*}
   \frac{ ka - \mathbb{E}[ka]}{\sqrt{\text{Var}(ka)}}  = 
   \frac{k (a - \mathbb{E}[a])}{k\sqrt{\text{Var}(a)}} 
 =
   \frac{ a - \mathbb{E}[a]}{\sqrt{\text{Var}(a)}}.
\end{gather*}
Thus, we have $f([\theta_{1}, \theta_{2}, \theta_{3}, \theta_{4}], \cdot) = f([k\theta_{1}, k\theta_{2}, \theta_{3}, \theta_{4}], \cdot)$. Now let $\Theta$ be the result of the internal \underline{\emph{direct sum}} $\Theta = \Theta' \oplus \Theta''$, and $\theta', \, \theta''$ be projections of $\theta$ onto the subspaces $\Theta' \,\&\, \Theta''$, respectively, so that $\theta = \theta' + \theta''$. The operator $+$ is defined as the vector sum, as usual. In the simplest case where $\Theta'$ and $\Theta''$ are aligned with the standard basis, this corresponds to  vectors in $\Theta'$ having zero valued entries in the place of parameters to which normalisation is applied, $\theta' = [0, 0, \theta_{3}, \theta_{4}]$. Vectors in $\Theta''$ are non-zero for normalised parameters, $\theta'' = [\theta_{1}, \theta_{2}, 0, 0]$.
Finally, we write the property of interest
\begin{gather*}
    f(\theta' + k \theta'', \cdot) = f(\theta' +  \theta'', \cdot).
\end{gather*}

\section{Additional proofs}

We use notation defined in \cref{sec:proofs} of the main text. Additionally, we denote by $D_x f(y)$ the directional derivative of the function $f$ in the direction of $x$, evaluated at the point $y$. This is the projection of the gradient of the function $f$ onto the vector $x$, given by
\begin{equation}
    D_x f(y) = \partial_z [f(z)](y) \cdot x.
\end{equation}

\subsection{Proof of \cref{lemma:scaling}}\label{appendix:scaling-proof}

\begin{proof} 
First, we consider $J'_k$ and $J''_k$. For $J'_k$, take any $\theta' \in \Theta'$ and consider the directional derivative $D_{\theta'} f(\tilde\theta' + k\tilde\theta'')$. From the limit definition,
\begin{align*}
    D_{\theta'} f(\tilde\theta' + k\tilde\theta'') 
    = \lim_{\delta \downarrow 0} \frac{1}{\delta}\left[f (\tilde\theta' + \delta\theta' + k\tilde\theta'' ), \cdot) - f (\tilde\theta' + k\tilde\theta''), \cdot) \right] 
    &= \lim_{\delta \downarrow 0} \frac{1}{\delta}\left[f (\tilde\theta' + \delta\theta' + \tilde\theta'' ), \cdot) - f (\tilde\theta' + \tilde\theta''), \cdot) \right] \\
    &= D_{\theta'} f(\tilde\theta' + \tilde\theta'').
\end{align*}

From the Jacobian-product definition, we have $J'_k \cdot \theta' = J' \cdot \theta'$. Since $\theta' \in \Theta'$ was arbitrary and we are working on a finite-dimensional Euclidean space, this shows $J'_k = J'$. For $J''_k$, consider $D_{\theta''} f(\tilde\theta' + k\tilde\theta'')$ for $\theta'' \in \Theta''$ arbitrary. We have
\begin{align*}
    D_{\theta''} f(\tilde\theta' + k\tilde\theta'') 
    &= \lim_{\delta \downarrow 0} \frac{1}{\delta}\left[f (\tilde\theta' + k\tilde\theta'' + \delta\theta'' ), \cdot) - f (\tilde\theta' + k\tilde\theta''), \cdot) \right] 
    = \lim_{\delta \downarrow 0} \frac{1}{\delta}\left[f (\tilde\theta' + \tilde\theta'' + \frac{\delta}{k} \theta''), \cdot) - f (\tilde\theta' + \tilde\theta''), \cdot) \right] \\
    &= \frac{1}{k}\lim_{\delta' \downarrow 0} \frac{1}{\delta'}\left[f (\tilde\theta' + \tilde\theta'' + \delta' \theta''), \cdot) - f (\tilde\theta' + \tilde\theta''), \cdot) \right] 
    = \frac{1}{k}D_{\theta''} f(\tilde\theta' + \tilde\theta'').
\end{align*}
Repeating the same argument as for $J'_k$, we obtain $J''_k = \frac{1}{k} J''$.

Now, we look at the scaling of $h_k$. By definition, using that $f$ is normalised and the previously derived scaling for $J_k$, 
\begin{align*}
    h_k(\theta' + k\theta'', \cdot) &= f(\tilde\theta' + k\tilde\theta'', \cdot) + J'_k (\theta' -\tilde\theta') + J''_k (k\theta'' - k\tilde\theta'')\\
    &= f(\tilde\theta' + \tilde\theta'', \cdot) + J'(\theta' -\tilde\theta') + J''(\theta'' -\tilde\theta'') \\
    &= h(\theta' + \theta'', \cdot),
\end{align*}
which is the claimed result.

For $H_k$, we examine it entry-wise. We have,
\begin{align*}
    [H_k]_{mn} 
    &= \partial_{\theta_m}\partial_{\theta_n} [L(h_k(\theta, \cdot))](\theta_k) 
    = \partial_{\theta_m} \left( \sum_i \partial_h \ell (h_k(\theta_k, x_i ), y_i) \cdot \partial_{\theta_n} [h(\theta, x_i)](\theta_k) \right) \\
    &= \sum_i \partial_{\theta_m} [h_k(\theta, x_i)](\theta_k) \cdot \partial^2_h \ell(h_k(\theta_k, x_i), y_i) \cdot \partial_{\theta_n} [h_k(\theta, x_i)](\theta_k)
    \\ &\quad\quad\quad+ \sum_j \partial_h \ell(h_k(\theta_k, x_j), y_j) \cdot \partial_{\theta_m} \partial_{\theta_n} [h_k(\theta, x_j)](\theta_k).
\end{align*}
Now since $h_k$ is affine, it has no curvature and thus $\partial_{\theta_m} \partial_{\theta_n} h_k(\theta, x)$ is identically zero for all $x \in \mcX$. With that, the second term in the sum vanishes. For the first sum, consider the \emph{middle term}, the curvature of the negative log-likelihood function, and use $h_k(\theta_k, \cdot) = h(\theta, \cdot)$ to see that it is invariant to $k$. Finally, note that $\partial_{\theta_m} h_k$ and $\partial_{\theta_n} h_k$ are entries of $J_k$ and inherit scaling from therein. Specifically, if both $\theta_m$ and $\theta_n$ belong to $\Theta''$, we obtain $1/k^2$ scaling; if just one belongs to $\Theta''$, we get $1/k$ scaling, and otherwise we obtain constant scaling. This completes the result for $H_k$.
\end{proof}

\subsection{Proof of equation~\cref{eq:directional-deriv-zero}}\label{appendix:proof_directional_derivative}

Consider a normalised network $f$, the linearisation point $\tilde \theta' + \tilde \theta''$,  and the directional derivative with respect to parameters $\tilde \theta''$ in the direction of $\tilde\theta''$, denoted $D_{\tilde\theta''} f(\tilde\theta' + \tilde\theta'')$. On one hand, this is just the projection of the partial derivative of $f$ with respect to $\theta''$ evaluated at $\tilde \theta''$ and projected onto $\tilde\theta$, and thus $D_{\tilde\theta''} f(\tilde\theta' + k\tilde\theta'') = \partial_{\theta''} f(\tilde \theta, \cdot) \cdot \tilde \theta$. On the other hand, from the limit definition of the directional derivative,
\begin{align*}
    D_{\tilde\theta''} f(\tilde\theta' + k\tilde\theta'') 
    = \lim_{\delta \downarrow 0} \frac{1}{\delta}\left[f (\tilde\theta' + \delta \tilde \theta'' + k\tilde\theta'' ), \cdot) - f (\tilde\theta' + k\tilde\theta''), \cdot) \right] 
    &= \lim_{\delta \downarrow 0} \frac{1}{\delta}\left[f (\tilde\theta' + (\delta+k) \: \tilde\theta'' ), \cdot) - f (\tilde\theta' + k \tilde\theta''), \cdot) \right] \\
    &=0,
\end{align*}
and thus $\partial_{\theta''} f(\tilde \theta, \cdot) \cdot \tilde \theta = 0$. 

The quantity appearing in equation~\cref{eq:directional-deriv-zero} is $\partial_{\theta''_{d:}} \phi(\tilde \theta''_{d:}, \cdot) \cdot \tilde \theta''_{d:}$. We now observe that for a fully normalised network, each of the outputs of the penultimate layer $[\phi(\theta_{d:}, \cdot)]_{i}$, with the output dimension being indexed by $i$, is a normalised network (\cref{def:normalised_networks}) in its own right. Hence, we can apply the same reasoning as above to see that $\partial_{\theta''_{d:}} [\phi(\tilde \theta''_{d:}, \cdot)]_{i} \cdot \tilde \theta''_{d:} = 0$.

\section{Algorithm for optimisation of the linear model loss}\label{app:algorithm}

We now discuss the optimisation of the loss for the predictor $h(v, \cdot) = J(\cdot) \cdot v$ where $v \in \Theta$ is the linear model's parameter vector. We note that the procedure for the Taylor expanded model $f(\tilde \theta, \cdot) + J(\cdot) \cdot (v - \tilde \theta)$ is analogous. We denote NN Jacobians as $J(\cdot) = \partial_\theta f(\tilde\theta, \cdot)$ and $d_{y}$ is the output dimensionality $|\mcY|$.

We wish to optimise $v$ according to the objective $\mathcal{L}_{h, \Lambda}(v) = L(h(v, \cdot)) + \|v\|^{2}_{\Lambda}$. We adopt a first order gradient-based approach. We first consider the gradient of $L(h(v, \cdot)) = \sum_{i}\ell( J(x_{i})\cdot v,  y_{i})$. Using the chain rule and evaluating at an arbitrary $\bar v \in \Theta$ we have
\begin{gather*}
    \partial_{v} [L(h(v, \cdot))](\bar v) = 
    \sum_i \partial_{\hat{y}} [\ell(\hat{y}, y)](J(x)\cdot \bar v) \cdot \partial_v (J(x_{i})\cdot \bar v) = \sum_i \partial_{\hat{y}} [\ell(\hat{y}, y)](J(x)\cdot \bar v) \cdot J(x_{i}).
\end{gather*}
Evaluating the affine function $h$ consists of computing the Jacobian vector product $J(x_i) v$. 
This can be done while avoiding computing the Jacobian explicitly by using forward mode automatic differentiation or finite differences.
We find both approaches to work similarly well, with finite differences being slightly faster, and forward mode automatic differentiation more numerically stable.
Our experiments use finite differences, so we present this approach here. Specifically, we employ the method of \citet{ANDREI2009Accelerated} to select the optimal step size. 
We then evaluate the loss gradient at the linear model output, denoting this vector in our algorithm as 
$ g = \partial_{\hat{y}} [\ell(\hat{y}, y)](J(x)\cdot \bar v)$. This gradient can often be evaluated in closed form. Finally, we project $g$ onto the weights by multiplying with the Jacobian. This vector Jacobian product is implemented using automatic differentiation. That is, $g^{T} J(x_{i})  =\partial_{\theta} [g^{T} \cdot f(\theta, x_{i})](\tilde \theta)$. We combine these steps in \cref{alg:linear_MAP_algorithm}.

Evaluating the gradient of $\|v\|^{2}_{\Lambda}$ is trivial. 

\begin{algorithm2e}[H]
	\SetAlgoLined
	\DontPrintSemicolon
	\KwData{Neural network $f$, Observation $x$,  Linearisation point $\tilde \theta$, Weights to optimise $v$, Likelihood function $\ell(\cdot, y)$, Machine precision $\epsilon$}
	$\delta = \sqrt{\epsilon} (1+\norm{\tilde \theta}_{\infty}) / \norm{v}_{\infty}$ \tcp*{Set FD stepsize \citep{ANDREI2009Accelerated}}
	$\hat{y} = J(x) \cdot v   \approx \frac{f(x, \tilde \theta + \delta v ) - f(x, \tilde \theta - \delta v ) }{2 \delta}$ \tcp*{Two sided FD approximation to Jvp}
	$g = \partial_{\hat{y}} [\ell(\hat{y}, y)](J(x)\cdot v)$ \tcp*{Evaluate gradient of loss at $J(x) \cdot v$ }
	$g^{T} \cdot J(x) = \partial_{\theta} [g^{T} \cdot f(\theta, x)](\tilde \theta) $ \tcp*{Project gradient with backward mode AD}
\KwResult{$g^{T} \cdot J(x)$}
\caption{Efficient evaluation of the likelihood gradient for the linearised model}
\label{alg:linear_MAP_algorithm}
\end{algorithm2e}

\section{Additional discussion of recommendations and results}\label{appendix:discussion}

Here, we expand on the motivation behind the recommendations made in the main text and the implications of our results.

\paragraph{Motivation for evaluating the log-likelihood Hessian at $\tilde\theta$ instead of $\theta_{\star}$}

In \cref{sec:adapt-laplace}, we suggest evaluating the curvature of the linear model's loss function only once and around the linearisation point $\tilde\theta$. This yields $H {=} \partial^2_\theta [L(h(\theta, \cdot))](\tilde\theta)$.

The more accurate alternative is to re-evaluate $\partial^2_\theta [L(h(\theta, \cdot))]$ at each optima of the linear model parameters found during the iterative procedure for computing $(\theta_\star, \Lambda_\star)$. This would require re-estimating the Hessian as many times as EM steps we perform.

However, we have good reason to believe that both approaches will yield similar results. The curvature of $L(h(\theta, \cdot))$ only depends on the linear model weights through the linear model's outputs, and the targets for likelihood functions outside of the exponential family. We assume our neural network model to be both flexible enough to fit the targets well and trained to do so. Furthermore, if we have a dense final layer, then the neural network's predictions $f(\tilde \theta, \cdot)$ are contained within the span of the linear model's basis expansion, as discussed in \cref{sec:dense-final-layer}. Then, we can the expect the mode of the linear model's posterior to also fit the training data well and thus provide similar predictions to the neural network. If both models provide similar predictions, then the curvature of $L(h(\theta, \cdot))$ at $\tilde \theta$ and $\theta_{\star}$ will be similar.

\paragraph{Implications of the simplification of the linear model for linearised networks \cref{eq:simple_linear_model}}

Both the Taylor expanded affine model presented in equation~\cref{eq:h-def} and the simplified linear model from equation~\cref{eq:simple_linear_model} employ the same Jacobian basis expansion. However, their posterior modes $\theta_{\star}$, and thus optimal regularisation parameters $\Lambda_{\star}$, are different. 
Specifically, the constant-in-$\theta$ offset present in the Taylor expanded model pushes its posterior mean closer to the linearisation point $\tilde \theta$. As discussed in \cref{sec:dense-final-layer}, scale invariant layers negate the effect of this offset. For fully normalised networks with a dense final layer, the offset is eliminated entirely and we recover the simplified linear model.

Accordingly, in our experiments, we find the simplified linear model's posterior mode to be further from the linearisation point than the Taylor expanded model's posterior mode. Thus, using the na\"ive linearised Laplace model evidence $\mcM_{\tilde \theta}$ results in poorer performance when using simplified linear model, i.e. when using normalisation layers. Even more interesting, however, is that employing the posterior mean of the simplified linear model in the model evidence estimate results in improved performance even when normalisation layers are not present, as shown in \cref{sec:experiments} and \cref{appendix:additional_results_arq}.

\paragraph{On the use of FixUp instead of scale-invariant normalisation}

FixUp regularisation is used by \citet{Immer2021Marginal} to avoid the need for normalisation layers in residual networks. FixUp specifies a network parameter initialisation scheme and adds a multiplicative scalar, and bias, at the output of each residual block. This multiplicative scalar introduces the following invariance; for a scaling of the weights in the residual block by $k > 0$ (applied as in \cref{sec:normalised-networks})  and a downscaling of the FixUp multiplicative scale factor by $\nicefrac{1}{k}$, the neural network output remains unchanged. We may employ methods analogous to the ones presented in the main text to study this scenario. Although this study is beyond the scope of this work, we make two observations.

\begin{enumerate}
    \item Unlike when employing normalisation layers, for a positive definite regulariser $\Lambda$, the invariance introduced by FixUp does not preclude the existence of a posterior mode nor does it result in the cancellation of terms in the affine Taylor expanded model \cref{eq:h-def} needed to recover the simplified linear model \cref{eq:simple_linear_model}. As a result, we expect the posterior mode of the linearised model for FixUp networks to be closer to the linearisation point $\tilde \theta$ than for normalised networks. In turn, a  na\"ive implementation of linearised Laplace will perform less poorly for FixUp networks than for normalised networks. Be that as it may, our results in \cref{sec:experiments} show there is still much to be gained from adopting \cref{rec:linear-weights}.
    
    \item Adept readers may note that a similar invariance is present in all neural networks with ReLU non-linearities; we can upscale the weights of a layer while downscaling the weights of any other layer and the neural network function is preserved.  When using a positive definite regulariser $\Lambda$, this invariance does not present a large problem for the Laplace method in practise. FixUp is unique, however, in that the multiplicative parameter is one dimensional. Since the dimensionality of the parameters belonging to the preceding weight layer is likely to be much larger than one, the regularisation term in the neural network loss function $\mcL_{f, \Lambda}$ biases the parameters towards solutions where the multiplicative scaling factor is large while other network parameters take values close to 0. This bias can be understood as an implicit choice of scaling factor $k$, defined as in \cref{sec:normalised-networks}, with value depending on the size of the weight layers in residual blocks.  
\end{enumerate}

\paragraph{On batch norm} 

Batch norm applies normalisation at the level of activations \citep{ioffe2015batch}. That is, each activation to which batch norm is applied is normalised using statistics about said activation computed across a batch of data. This is distinct from other approaches to normalisation, such as group norm and layer norm, which compute statistics across network activations but separately for each individual observation.   Consequently, batch norm introduces two peculiarities that are not shared by other normalisation layers.
\begin{enumerate}
	\item When using batch norm, the parameters feeding into each activation are normalised separately. As a result, \cref{rec:factorised_prior} prescribes regularising the weights feeding into each activation separately. However, empirically, we  find that doing this provides no benefit over defining per-layer regularisation parameters. We hypothesize this is due to the scaling constant $k$ (defined as in \cref{sec:normalised-networks}) being similar for weights corresponding to the same layers, perhaps due to the gradients reaching the different weights in each layer having similar variance \citep{Laarhoven2017L2}. Thus, our experiments with batch norm implement layer-wise regularisation instead of activation-wise regularisation.
	
	\item In the presence of batch normalisation, including network biases' gradients into the Jacobian feature expansion is not necessary in order to preserve the invariances described in \cref{sec:normalised-networks}. Because network biases represent a scalar offset to network activations, their effect is completely negated by batch normalisation. Indeed, for this reason, biases are most commonly omitted from weight layers followed by batch norm \citep{he2016deep}. 
\end{enumerate}



\paragraph{On the use of weight decay and data augmentation for neural network training}

A common concern when performing probabilistic reasoning with neural networks is whether the regularisation techniques used for training break the formalisms put in place for the purpose of inference. For instance, data augmentation breaks any assumptions of observations being iid put in place for the speficication of a likelihood function \citep{Nabarro2021augmentation}. Conversely, some of the techniques required to satisfy the probabilistic inference formalism may degrade network training performance. This is the case for weight decay, which appears as the log of a zero-mean Gaussian prior over the network's parameters but can hinder learning in generative models.

The adaption of linearised Laplace presented in \cref{sec:adapt-laplace} of this paper provides a simple solution to the above qualms by decoupling neural network training from probabilistic inference. Applying \cref{rec:linear-weights} removes the need for the linearisation point $\tilde \theta$ to be a stationary point of the neural network posterior distribution. Thus, we are free to employ weight decay, a lack there-off, data augmentation, or any other regularisation schemes. In turn, properties introduced into our neural networks by training with these regularisation schemes will feature in the resulting linear model through the Jacobian basis expansion. 

When performing inference in the linear model, we should take care to use a proper likelihood function and prior however. For this reason, we avoid data augmentation when computing the optima of the linear model loss $\theta_{\star}$ and when computing the linear model Hessian $H$.

\section{Details on experimental setup}\label{app:experimental_setup}

Here, we provide the details of our experimental setup which were omitted from the main text.

\subsection{Experiments with full Hessian computation}\label{app:experimental_setup_small}

This subsection concerns the experiments which use small architectures for which exact Hessian computation is tractable. These experiments are described in \cref{subsec:exp_assumptions} and \cref{subsec:architectures} of the main text. We first describe the setup components shared among architectures and then provide architecture-specific details. We exclude details for the U-net used in \cref{subsec:architectures}. Instead we provide these together with a brief description of the tomographic reconstruction task it performs in \cref{appendix:unet-setup}.

Unless specified otherwise, NN weights $\tilde \theta$ are learnt using SGD, with an initial learning rate of 0.1, momentum of 0.9, and weight decay of $1\times 10^{-4}$. We trained for 90 epochs, using a multi-step LR scheduler with a decay rate of 0.1 applied at epochs 40 and 70. This is a standard  choice for CNNs and is default in the examples provided by \href{https://pytorch.org/vision/stable/models.html}{Pytorch}.

The linear weights $\theta_\star$ are optimised using Adam and with their gradients calculated using \cref{alg:linear_MAP_algorithm}. We use a learning rate of $1\times 10^{-4}$ and train for 100 epochs. We set the initial regularisation parameter to be isotropic $\Lambda = \lambda I$ with $\lambda = 1\times 10^{-4}$.

\subsubsection{CNN}

Our CNN is based on the LeNet architecture with a few variations found in more modern neural networks. The architecture contains 3 convolutional blocks, followed by global average pooling in the spatial dimensions, a flatten operation, and finally a fully-connected layer. The convolutional blocks consist of \textsc{Conv} $\rightarrow$ \textsc{ReLU} $\rightarrow$ \textsc{BatchNorm}. Instead of using max pooling layers, as in the original LeNet variants, we use convolutions with a stride of 2. The first convolution is $5 \times 5$, while the next two are $3 \times 3$. As described in the main text, we consider architectures of 3 different sizes. \cref{tab:lenet_models} shows the number of the filters and number of parameters for each size of this model. The \emph{Big} model's values where chosen to create a model as large as possible while keeping full-covariance Laplace inference tractable on one A100 GPU.

\begin{table}[h]
    \caption{Architecture parameters for the CNNs used in experiments.}
    \label{tab:lenet_models}
    \centering
    \begin{tabular}{cccccc}
    \toprule
         &  \textsc{Conv1 Filters} &  \textsc{Conv2 Filters} &  \textsc{Conv3 Filters} & \textsc{Params.} & \textsc{Hessain Size} \\ \midrule
     Big & 42 & 48 & 60 & 46 024 & 15.68 GB \\
     Med & 32 & 32 & 64 & 29 226 & 6.36 GB \\
     Small & 16 & 32 & 32 & 14 634 & 1.60 GB \\
     \bottomrule
    \end{tabular}
\end{table}

\subsubsection{ResNet, Pre-ResNet, and Biased-ResNet}\label{appendix:ResNet_arq}

Our ResNet is based on our CNN architecture. We replace the second and third convolutional blocks with residual blocks. The main branch of the residual blocks consist of \textsc{Conv} $\rightarrow$  \textsc{BatchNorm} $\rightarrow$ \textsc{ReLU} $\rightarrow$ \textsc{Conv} $\rightarrow$  \textsc{BatchNorm}. We apply a final \textsc{ReLU} layer after the residual is added. All of the convolutions in the residual blocks use the same number of filters. In order to downsample our features between blocks we use $1 \times 1$ convolutions with a stride of 2. \cref{tab:resnet_models} shows the number of the filters and number of parameters used for each size of this model.

Our Pre-ResNet architecture is identical to the ResNet except the main branch cosists of \textsc{BatchNorm} $\rightarrow$ \textsc{ReLU} $\rightarrow$ \textsc{Conv} $\rightarrow$  \textsc{BatchNorm} $\rightarrow$ \textsc{ReLU} $\rightarrow$ \textsc{Conv}, and we do not apply a \textsc{ReLU} after adding the residual.

\begin{table}[h]
    \caption{Architecture parameters for the ResNets used in experiments.}
    \label{tab:resnet_models}
    \centering
    \begin{tabular}{cccccc}
    \toprule
         &  \textsc{Conv1 Filters} &  \textsc{ResBlock 1 Filters} &  \textsc{ResBlock 2 Filters} & \textsc{Params.} & \textsc{Hessain Size} \\ \midrule
     Big & 22 & 42 & 64 & 45 576 & 15.48 GB \\
     Med & 16 & 32 & 64 & 26 874 & 5.38 GB \\
     Small & 12 & 24 & 32 & 14 814 & 1.64 GB \\
     \bottomrule
    \end{tabular}
\end{table}

Note that the standard ResNet architecture does not apply biases in the convolution layers. The only biases in the entire network are placed in the dense output layer. For our experiment where biases are included in the Jacobian feature expansion in \cref{subsec:exp_assumptions}, we modify the ResNet architecture to include biases in all convolutional layers in addition to the already present final dense layer bias. These biases account for a small increase in parameters, to 14 898, 26 986, and 45 726, in the small, medium, and big cases, respectively.

\subsubsection{FixUp ResNet}

Our FixUp-ResNet architecture follows the standard ResNet structure described above, with the additional FixUP offsets and multipliers described in \citep{Zhang19FixUp}. We also follow \citet{Zhang19FixUp} in zero initializing the dense layer, and scaling the convolution weight initalization as a function of the depth of the network.

When training FixUp-ResNets, we use the Adam optimizer with a fixed learning rate of 0.01.

\subsubsection{Transformer}

Our Transformer architecture contains two encoder layers with two attention heads each, and no dropout. Its input is a sequence of tokens, to which we apply a linear embedding. We add a learnable \texttt{class} embedding for each input. This \texttt{class} token is used to classify the input. We do not use positional encoding, preserving permutation invariance in the input. The sizes of the embeddings and the MLP hidden dimensions are provided in \cref{tab:transformer_models}.

When training Transformers, we use the Adam optimiser with a learning rate of $3 \times 10^{-3}$. We use an exponential learning rate decay with a gamma of 0.99 applied after every epoch of training.

\begin{table}[h]
    \caption{Architecture parameters for the Transformers used in experiments.}
    \label{tab:transformer_models}
    \centering
    \begin{tabular}{cccccc}
    \toprule
         &  \textsc{MLP Dim} &  \textsc{Embedding Dim} &  \textsc{Params.} & \textsc{Hessain Size} \\ \midrule
     Big & 120 & 50  & 45 900 & 15.70 GB \\
     Med & 80 & 40  & 27 090 & 5.47 GB \\
     Small & 60 & 30  & 15 520 & 1.79 GB \\
     \bottomrule
    \end{tabular}
\end{table}

\subsection{U-Net tomographic reconstruction of KMNIST digits}\label{appendix:unet-setup}

In this section, we provide an overview of the tomographic reconstruction task for which we report results in \cref{subsec:architectures}. We also report the corresponding experimental details. Our setup almost exactly replicates that of \citet{barbano2021dip} and \citet{antoran2022tomography}, and we refer to these works for a fully detailed description.

Tomographic reconstruction aims to recover an unknown image $z\in\mathbb{R}^{d_{z}}$ from noisy measurements $y\in \mcY$, which are often described by the application of a linear forward operator $\mathrm{A}\in\mathbb{R}^{d_{y}\times d_{z}}$ (in our case $\mathrm{A}$ is the discrete Radon Transform) and the addition of Gaussian noise $\eta\sim\mathcal{N}(0,\sigma^{2}_{y}\,\mathrm{I}_{d_{y}})$ as
\begin{equation}\label{eq:inverse_problem}
    y = \mathrm{A} z + \eta,
\end{equation}
and $|\mcY| \ll d_{z}$.
We replicate the setup of \citet{barbano2021dip,antoran2022tomography}. That is, we 
use 10 test images from the KMNIST dataset, which consists of $28\times 28$ gray-scale images of Hiragana characters \citep{clanuwat2018deep}, we simulate $y$ with $20$ angles taken uniformly from the range $0^\circ$ to $180^\circ$, and add $5\%$ white noise to $\mathrm{A}z$.
We reconstruct $z$ using the Deep Image Prior (DIP) \citep{Ulyanov18prior}, which parametrises the reconstruction $z$ as the output of a U-net $f(\theta)$ \citep{ronneberger2015u}.
The reconstruction loss used to train the U-net reads
\begin{equation*}
     \|\mathrm{A} f(\theta) - y\|^2 + \zeta \mathrm{TV}(f(\theta)),
\end{equation*}
where the hyperparameter $\zeta>0$ determines the strength of the Total Variation (TV) regularisation \citep{chambolle2004algorithm}. This regulariser is used to train the U-net but then substituted for the Gaussian regulariser described in the main text for the purpose of inference in the linearised network. As discussed in \cref{appendix:discussion}, this results in a valid inference scheme when \cref{rec:linear-weights} is applied.

We use the U-net like architecture deployed by \citet{barbano2021dip}. Group norm is placed before every $3\times 3$ convolution operation. The U-net architecture is a encoder-decoder, fully convolutional deep model constructing multi-level feature maps.
We identify 3 distinct blocks for both the encoder branch and and 2 blocks for the the decoder branch: \texttt{In}, \texttt{Down}$_0$, \texttt{Down}$_{1}$, and \texttt{Up}$_{0}$ and \texttt{Up}$_{1}$, respectively. 
The \texttt{In} block consists of a $3\times 3$ convolution. 
\texttt{Down} blocks consist of a $3\times 3$ convolution with stride of 2 followed by a $3\times 3$ convolution operation and a bi-linear up-sampling. 
The \texttt{Up} blocks instead consist of two successive $3\times 3$ convolutional operations. Given the use of the leaky ReLU non-linearity, the normalised parameter groups of this network coincide with the described blocks.
The number of channels is set to 32 at every scale.
Multi-channel feature maps from the \texttt{In} block and from \texttt{Down}$_{0}$ are first transformed via a $1\times 1$ convolutional operation to 4 channel feature maps and then fed to \texttt{Up}$_{1}$, \texttt{Up}$_{0}$.
The reconstructed image is obtained as the output of \texttt{Up}$_{1}$ further processed via a $1\times 1$ convolutional layer.
The total number of parameters is 78k.
This is too many for full Hessian construction on GPU but we get around this issue by performing inference in the lower dimensional space of observations, as described below.
We refer to \citet{antoran2022tomography} for a full list of hyperparameters involved in training the U-net.



We construct a data-driven prior for the reconstruction $z$ by adopting a linear model on the basis expansion given by the Jacobian of the trained U-net. However, differently from the procedure described in \cref{sec:preliminaries}, we leverage the weight-space function-space duality of linear models \citep{RasmussenW06} to formulate our prior in the space of reconstructions directly
\begin{gather*}
    z = J \theta,\quad \theta\sim \mathcal{N}(0, \Lambda^{-1}) \quad \text{and thus} \quad z \sim \mathcal{N}(0, J \Lambda^{-1} J^{\top}).
\end{gather*}
This formulation allows us to perform (full-Hessian) Laplace inference at cost scaling cubically in $|\mcY|$, which is much smaller that $|\Theta|$ for this task. We refer to \citet{antoran2022tomography} for details of the inference procedure employed.

The prior covariance $\Lambda^{-1}$ is a filter-wise block-diagonal matrix which applies separate regularisation to the parameters of each block in the U-net, as described by \citet{barbano2021dip} and \citet{antoran2022tomography}.
For the single regulariser experiment, we keep the same prior structure but tie the marginal prior variance of all parameters. That is, we ensure all entries of the diagonal of $\Lambda^{-1}$ are the same. The parameters of these regularisers are learnt via model evidence optimisation, as described in the main text. 

\subsection{Large scale experiments}\label{app:experimental_setup_big}

For scaling linearised Laplace to large ResNets (i.e. ResNet-18 with 11M parameters and ResNet-50 with 23M parameters), we employ a Kronecker-factorisation of the Hessian/GGN. This is a common way to scale the Laplace approximation to large models \citep{daxberger2021laplace} and was originally proposed in \citet{ritter2018scalable}. We use the recently-released \texttt{laplace} library\footnote{\url{https://github.com/AlexImmer/Laplace}} \citep{daxberger2021laplace} for fitting the KFAC Laplace models. For ResNets with batch norm, we use the reference implementation from the \texttt{torchvision} package\footnote{\url{https://pytorch.org/vision/stable/models.html}}. For ResNets with fixUp, we use a popular open-source implementation\footnote{\url{https://github.com/hongyi-zhang/FixUp}}. To train the ResNet parameters, which will be used as the linearisation points $\tilde \theta$, we use the same hyperparameters as described at the top of \cref{app:experimental_setup_small} for both batch norm and FixUp ResNets.

Finally, in \cref{fig:kfac}, we show that for the big ResNet model from \cref{tab:resnet_models} (i.e. the largest model for which we can tractably compute the full linear model Hessian), the optima of the model evidence induced by the KFAC approximation is similar to that obtained when using the full Hessian. This justifies the usage of KFAC Laplace for model evidence estimation with the full-scale ResNet models (where using the full Hessian is intractable).

\begin{figure}[htb]
    \centering
    \includegraphics[width=0.5\textwidth]{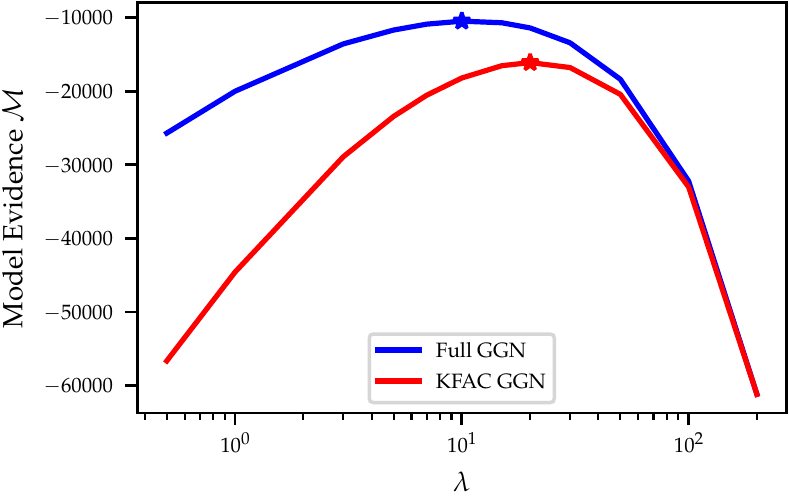}
    \caption{Linear model evidence for the big ResNet (cf.~\cref{tab:resnet_models}) plotted as a function of a single regularisation strength parameter $\lambda$, with $\Lambda = \lambda I$, for both the full Hessian and its KFAC approximation.}
    \label{fig:kfac}
\end{figure}

\clearpage

\section{Additional experimental results}\label{app:additional_results}

In this appendix we provide additional experimental results and analyses that are not included in the main text.

\subsection{Hessian choice}

\cref{fig:app_hess_lamb} shows the differences in the optimised layer-wise regularisation strength values when using the na\"ive objective $\mcM_{\tilde \theta}$ and the recommended objective $\mcM_{\theta_{\star}}$. We also compare using the Hessian of the linear model's loss evaluated at $\tilde \theta$ and at $\theta_{\star}$. The choice of hyperparameter optimisation objective has a very large impact on the resulting regularisation strength, with the na\"ive objective leading to much smaller regularisation strengths. As a result, the linear model's errorbars become too large and test performance deteriorates, as shown in \cref{fig:hessian_bias_choice} in the main text. On the other hand, the choice of Hessian evaluation location has negligible impact on the learnt hyperparameters. Intuitively, the basis functions corresponding to weights closer to the network output are regularised less, since these features are more predictive of the targets. For this experiment we used the 46k parameter ResNet described in \cref{appendix:ResNet_arq} and the MNIST dataset.

\begin{figure}[htbp]
    \centering
    \includegraphics{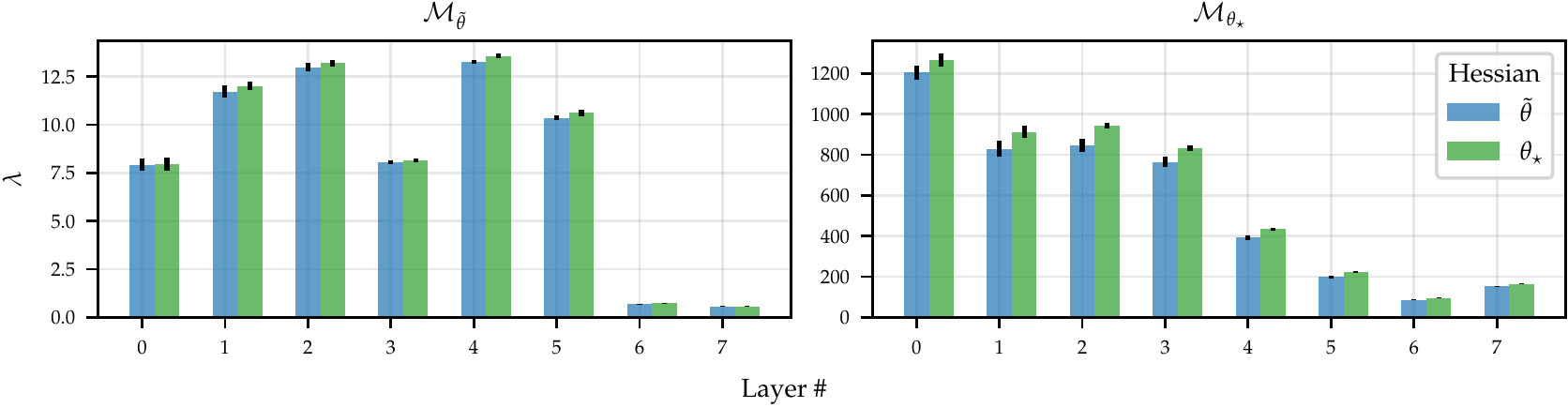}
    \caption{Layer-wise precision values learnt in the Hessian choice experiment. As before, we see that the choice of location around which the Hessian is calculated has a very small impact and the choice of parameters featuring in the norm term. The choice of na\"ive or recommended model evidence objective has a large impact. We note the much larger scale of the precisions in the case of the linear model weights.}
    \label{fig:app_hess_lamb}
\end{figure}

\subsection{Treatment of NN biases}

\cref{fig:app_bayes_lamb} shows the change in layer-wise regularisation strength values obtained when excluding the NN biases from the Jacobian basis expansion. The impact from the exclusion of biases is negligible while the application of \cref{rec:linear-weights} results in a difference of roughly two orders of magnitude in the obtained regularisation strengths. 

\begin{figure}[h]
    \centering
    \includegraphics{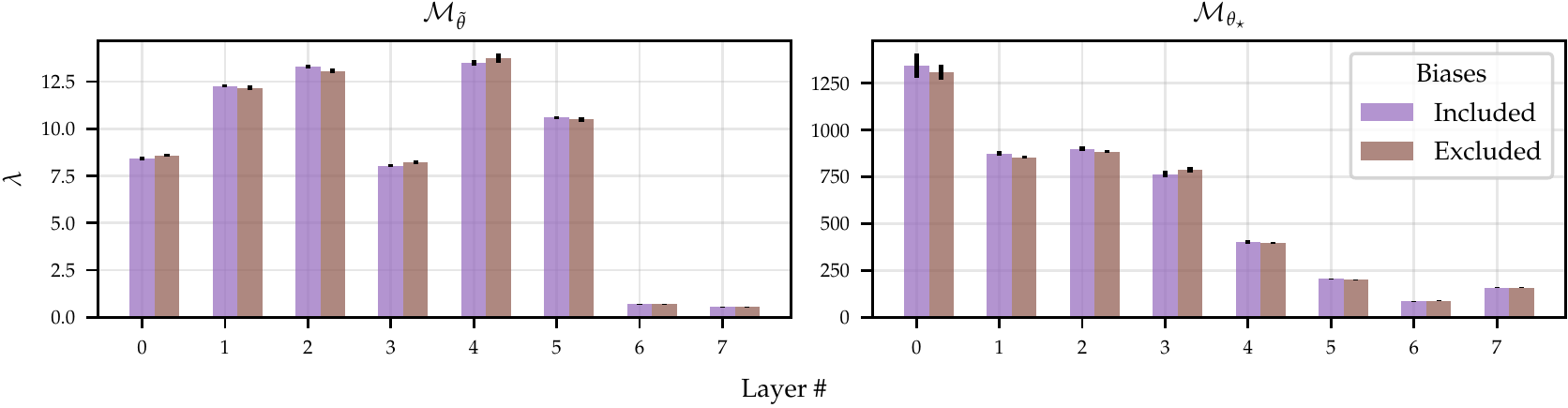}
    \caption{Layer-wise regularisation strengths learnt in the bias experiment of \cref{fig:hessian_bias_choice}. 
    As before, we see that the choice of bias treatment has a smaller effect than the choice of weights featured in the norm of the model evidence. We also note the much larger scale of the precisions in the case of the linear model weights.
    }
    \label{fig:app_bayes_lamb}
\end{figure}

\cref{fig:app_bayes_weights} compares the empirical distributions of the NN linearisation point $\tilde \theta$ and linear model optima parameters $\theta_{\star}$. The parameters of the NN linearisation point present much larger variance than those of the linear model optima.
Once again, we see that the exclusion of biases has little impact.

For this experiment we used the 46k parameter ResNet described in \cref{appendix:ResNet_arq} and the MNIST dataset.

\begin{figure}[h]
    \centering
    \includegraphics{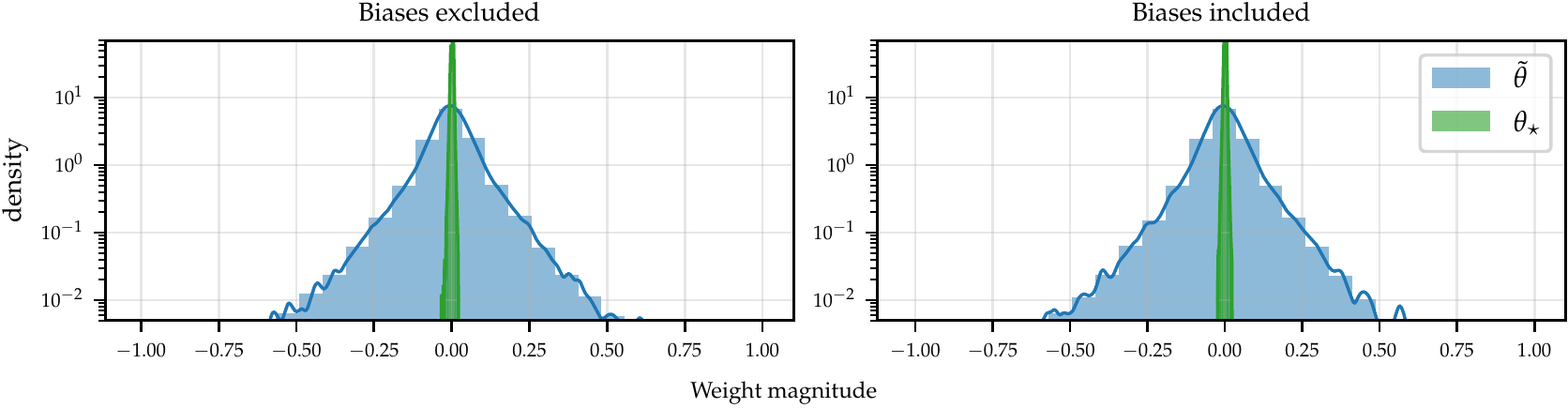}
    \caption{Histograms of the individual entries of $\tilde \theta$ and $\theta_{\star}$ for the models in the bias exclusion experiment of \cref{fig:hessian_bias_choice}.}
    \label{fig:app_bayes_weights}
\end{figure}

\subsection{Early stopping}

\cref{fig:app_early_stopping} shows the evolution of the layerwise optimised regularisation strength values for the Jacobian basis corresponding to the NN at different stages of training. We compare the regularisation parameters found with and without applying \cref{rec:linear-weights} and with and without batch norm. The results of \cref{fig:early_stopping} are confirmed by the fact that that additional training of the NN does not cause the regulariser strengths found with $\mcM_{\theta_{\star}}$ and $\mcM_{\tilde \theta}$ to become closer to each other.

\begin{figure}[htbp]
    \centering
    \includegraphics{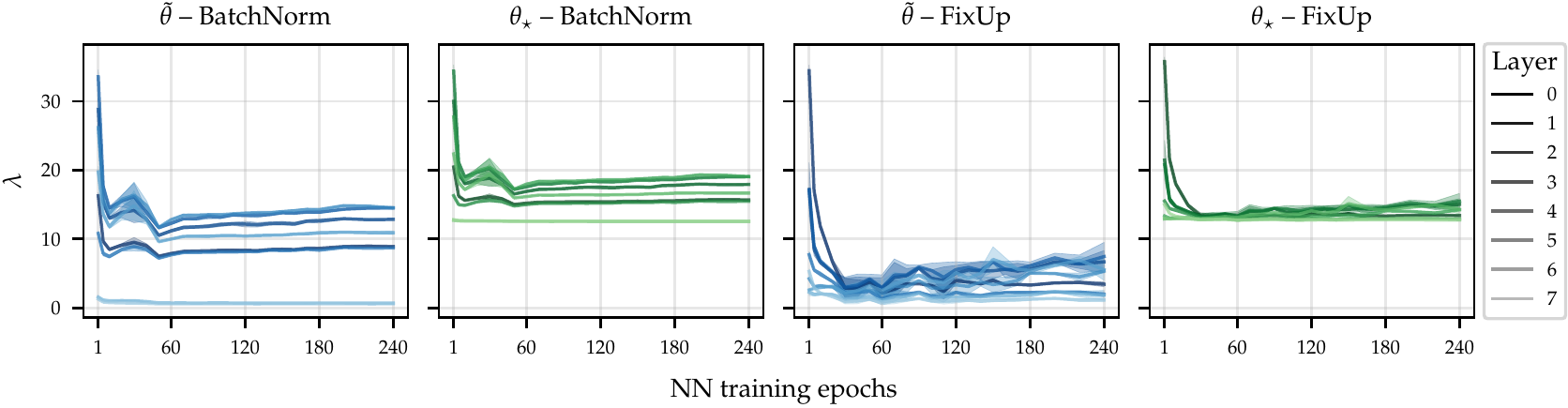}
    \caption{Evolution of precisions for the no early stopping experiment in \cref{fig:early_stopping}. In both the batch norm and fixUp cases, we see that as we train the NN for longer, the regularisation values obtained with $\mcM_{\theta_{\star}}$ and $\mcM_{\tilde \theta}$ converge to different values.}
    \label{fig:app_early_stopping}
\end{figure}

\subsection{Validation of recommendations across architectures}\label{appendix:additional_results_arq}

\cref{tab:mega_model_sweep} extends the results of \cref{tab:model_sweep} to two smaller sized models (\emph{Med} and \emph{Small}). We exclude the U-net architecture for which we only run a single model size with results already shown in \cref{tab:model_sweep}.
We also compare a number of schemes for obtaining the parameters that optimise of the linear model's loss.

Specifically, for each class of model, we compare the simplified linear model (i.e. \cref{eq:simple_linear_model}) optima $\theta_{*,\text{simple}}$ to the Taylor expanded model (i.e. \cref{eq:h-def}) optima $\theta_{*,\text{full}}$. We find that $\theta_{*,\text{simple}}$ performs equal to or (more often) better than $\theta_{*,\text{full}}$. For fully normalised networks, both approaches should be equivalent. However, the simplified model allows for more numerically stable optimisation of the weights. We can gain intuition for this by noting that fitting the Taylor expanded model's parameters can be viewed as regressing onto the neural network's residuals, which are likely to resemble noise. The more accurate parameter optima estimate obtained with the simplified linear model in turn yields a more accurate estimate of the model evidence objective. 

For non-normalised FixUp models, the Taylor expanded model does not simplify, making the superior performance of the simplified objective intriguing.  
We conjecture that this is again a result of the simplified linear model being easier to optimise. We attempted to solve this problem by initialising the Taylor expanded model's weights at the linearisation point $\tilde \theta$. We denote the results obtained with a model evidence featuring a set of parameters in its norm term found in this way as $\theta_{*,\text{full,NN}}$. This tends to perform worse than initialising the Taylor expanded model's weights at 0, the default approach used in all other settings. This result tells us that the Taylor expanded model's parameter optima is likely closer to that of simplified linear model than to the NN linearisation point. However, we can not discard the optimisation instability hypothesis. 

As before, \cref{rec:linear-weights} (using $\mathcal{M}_{\theta_*}$ rather than $\mathcal{M}_{\tilde\theta}$) yields notable improvements in most cases. Also in most cases, \cref{rec:factorised_prior} (using per-layer prior precisions in normalised networks) yields moderate improvements or does not degrade performance for normalised networks. 

\begin{table}[htbp]
    \centering
    \caption{Extension of \cref{tab:model_sweep} to medium and small sized models, as well as to different methods for optimisation of the linear models' parameter optima.}
    \label{tab:mega_model_sweep}
    \begin{tabular}{llllllll}
\toprule
      &                                &                     &                       \textsc{Transformer} &                            \textsc{CNN} &                         \textsc{ResNet} &                     \textsc{Pre-ResNet} &                          \textsc{FixUp} \\
\midrule
\multirow{8}{*}{Big} & \multirow{2}{*}{$\mathcal{M}_{\theta_{\star, \text{simple}}}$} & single $\lambda$ &  0.162 {\tiny \color{gray} $\pm$ 0.042} &  0.025 {\tiny \color{gray} $\pm$ 0.000} &  0.017 {\tiny \color{gray} $\pm$ 0.000} &  0.017 {\tiny \color{gray} $\pm$ 0.000} &  0.020 {\tiny \color{gray} $\pm$ 0.001} \\
      &                                & multiple $\lambda$s &  0.162 {\tiny \color{gray} $\pm$ 0.042} &  0.025 {\tiny \color{gray} $\pm$ 0.000} &  0.016 {\tiny \color{gray} $\pm$ 0.001} &  0.016 {\tiny \color{gray} $\pm$ 0.000} &  0.018 {\tiny \color{gray} $\pm$ 0.000} \\
\cmidrule{2-8}
      & \multirow{2}{*}{$\mathcal{M}_{\theta_{\star, \text{full}}}$} & single $\lambda$ &  0.161 {\tiny \color{gray} $\pm$ 0.042} &  0.073 {\tiny \color{gray} $\pm$ 0.001} &  0.083 {\tiny \color{gray} $\pm$ 0.003} &  0.087 {\tiny \color{gray} $\pm$ 0.001} &  0.055 {\tiny \color{gray} $\pm$ 0.006} \\
      &                                & multiple $\lambda$s &  0.162 {\tiny \color{gray} $\pm$ 0.042} &  0.073 {\tiny \color{gray} $\pm$ 0.001} &  0.066 {\tiny \color{gray} $\pm$ 0.002} &  0.064 {\tiny \color{gray} $\pm$ 0.002} &  0.061 {\tiny \color{gray} $\pm$ 0.005} \\
\cmidrule{2-8}
      & \multirow{2}{*}{$\mathcal{M}_{\theta_{\star, \text{full,NN}}}$} & single $\lambda$ &  0.161 {\tiny \color{gray} $\pm$ 0.042} &  0.211 {\tiny \color{gray} $\pm$ 0.002} &  0.149 {\tiny \color{gray} $\pm$ 0.004} &  0.117 {\tiny \color{gray} $\pm$ 0.002} &  0.060 {\tiny \color{gray} $\pm$ 0.006} \\
      &                                & multiple $\lambda$s &  0.161 {\tiny \color{gray} $\pm$ 0.042} &  0.169 {\tiny \color{gray} $\pm$ 0.002} &  0.138 {\tiny \color{gray} $\pm$ 0.004} &  0.109 {\tiny \color{gray} $\pm$ 0.003} &  0.070 {\tiny \color{gray} $\pm$ 0.011} \\
\cmidrule{2-8}
      & \multirow{2}{*}{$\mathcal{M}_{\tilde{\theta}}$} & single $\lambda$ &  0.310 {\tiny \color{gray} $\pm$ 0.060} &  0.253 {\tiny \color{gray} $\pm$ 0.001} &  0.252 {\tiny \color{gray} $\pm$ 0.006} &  0.220 {\tiny \color{gray} $\pm$ 0.004} &  0.153 {\tiny \color{gray} $\pm$ 0.021} \\
      &                                & multiple $\lambda$s &  0.162 {\tiny \color{gray} $\pm$ 0.042} &  0.205 {\tiny \color{gray} $\pm$ 0.002} &  0.236 {\tiny \color{gray} $\pm$ 0.005} &  0.239 {\tiny \color{gray} $\pm$ 0.004} &  0.200 {\tiny \color{gray} $\pm$ 0.018} \\
\cmidrule{1-8}
\cmidrule{2-8}
\multirow{8}{*}{Med} & \multirow{2}{*}{$\mathcal{M}_{\theta_{\star, \text{simple}}}$} & single $\lambda$ &  0.132 {\tiny \color{gray} $\pm$ 0.026} &  0.024 {\tiny \color{gray} $\pm$ 0.000} &  0.017 {\tiny \color{gray} $\pm$ 0.000} &  0.018 {\tiny \color{gray} $\pm$ 0.000} &  0.017 {\tiny \color{gray} $\pm$ 0.000} \\
      &                                & multiple $\lambda$s &  0.133 {\tiny \color{gray} $\pm$ 0.026} &  0.023 {\tiny \color{gray} $\pm$ 0.000} &  0.016 {\tiny \color{gray} $\pm$ 0.000} &  0.017 {\tiny \color{gray} $\pm$ 0.000} &  0.038 {\tiny \color{gray} $\pm$ 0.021} \\
\cmidrule{2-8}
      & \multirow{2}{*}{$\mathcal{M}_{\theta_{\star, \text{full}}}$} & single $\lambda$ &  0.132 {\tiny \color{gray} $\pm$ 0.025} &  0.060 {\tiny \color{gray} $\pm$ 0.001} &  0.059 {\tiny \color{gray} $\pm$ 0.002} &  0.055 {\tiny \color{gray} $\pm$ 0.002} &  0.062 {\tiny \color{gray} $\pm$ 0.003} \\
      &                                & multiple $\lambda$s &  0.132 {\tiny \color{gray} $\pm$ 0.026} &  0.059 {\tiny \color{gray} $\pm$ 0.001} &  0.050 {\tiny \color{gray} $\pm$ 0.001} &  0.047 {\tiny \color{gray} $\pm$ 0.000} &  0.079 {\tiny \color{gray} $\pm$ 0.024} \\
\cmidrule{2-8}
      & \multirow{2}{*}{$\mathcal{M}_{\theta_{\star, \text{full,NN}}}$} & single $\lambda$ &  0.132 {\tiny \color{gray} $\pm$ 0.025} &  0.154 {\tiny \color{gray} $\pm$ 0.001} &  0.097 {\tiny \color{gray} $\pm$ 0.002} &  0.074 {\tiny \color{gray} $\pm$ 0.002} &  0.072 {\tiny \color{gray} $\pm$ 0.005} \\
      &                                & multiple $\lambda$s &  0.132 {\tiny \color{gray} $\pm$ 0.025} &  0.124 {\tiny \color{gray} $\pm$ 0.002} &  0.087 {\tiny \color{gray} $\pm$ 0.001} &  0.072 {\tiny \color{gray} $\pm$ 0.002} &  0.065 {\tiny \color{gray} $\pm$ 0.003} \\
\cmidrule{2-8}
      & \multirow{2}{*}{$\mathcal{M}_{\tilde{\theta}}$} & single $\lambda$ &  0.241 {\tiny \color{gray} $\pm$ 0.028} &  0.190 {\tiny \color{gray} $\pm$ 0.002} &  0.165 {\tiny \color{gray} $\pm$ 0.005} &  0.132 {\tiny \color{gray} $\pm$ 0.004} &  0.183 {\tiny \color{gray} $\pm$ 0.008} \\
      &                                & multiple $\lambda$s &  0.133 {\tiny \color{gray} $\pm$ 0.025} &  0.153 {\tiny \color{gray} $\pm$ 0.002} &  0.176 {\tiny \color{gray} $\pm$ 0.004} &  0.167 {\tiny \color{gray} $\pm$ 0.001} &  0.402 {\tiny \color{gray} $\pm$ 0.213} \\
\cmidrule{1-8}
\cmidrule{2-8}
\multirow{8}{*}{Small} & \multirow{2}{*}{$\mathcal{M}_{\theta_{\star, \text{simple}}}$} & single $\lambda$ &  0.112 {\tiny \color{gray} $\pm$ 0.016} &  0.025 {\tiny \color{gray} $\pm$ 0.000} &  0.017 {\tiny \color{gray} $\pm$ 0.001} &  0.018 {\tiny \color{gray} $\pm$ 0.001} &  0.018 {\tiny \color{gray} $\pm$ 0.000} \\
      &                                & multiple $\lambda$s &  0.112 {\tiny \color{gray} $\pm$ 0.016} &  0.024 {\tiny \color{gray} $\pm$ 0.000} &  0.018 {\tiny \color{gray} $\pm$ 0.000} &  0.018 {\tiny \color{gray} $\pm$ 0.000} &  0.016 {\tiny \color{gray} $\pm$ 0.001} \\
\cmidrule{2-8}
      & \multirow{2}{*}{$\mathcal{M}_{\theta_{\star, \text{full}}}$} & single $\lambda$ &  0.112 {\tiny \color{gray} $\pm$ 0.016} &  0.045 {\tiny \color{gray} $\pm$ 0.001} &  0.040 {\tiny \color{gray} $\pm$ 0.001} &  0.040 {\tiny \color{gray} $\pm$ 0.001} &  0.055 {\tiny \color{gray} $\pm$ 0.003} \\
      &                                & multiple $\lambda$s &  0.112 {\tiny \color{gray} $\pm$ 0.016} &  0.045 {\tiny \color{gray} $\pm$ 0.000} &  0.038 {\tiny \color{gray} $\pm$ 0.000} &  0.035 {\tiny \color{gray} $\pm$ 0.001} &  0.051 {\tiny \color{gray} $\pm$ 0.001} \\
\cmidrule{2-8}
      & \multirow{2}{*}{$\mathcal{M}_{\theta_{\star, \text{full,NN}}}$} & single $\lambda$ &  0.112 {\tiny \color{gray} $\pm$ 0.016} &  0.085 {\tiny \color{gray} $\pm$ 0.001} &  0.059 {\tiny \color{gray} $\pm$ 0.001} &  0.048 {\tiny \color{gray} $\pm$ 0.001} &  0.060 {\tiny \color{gray} $\pm$ 0.003} \\
      &                                & multiple $\lambda$s &  0.112 {\tiny \color{gray} $\pm$ 0.016} &  0.075 {\tiny \color{gray} $\pm$ 0.001} &  0.054 {\tiny \color{gray} $\pm$ 0.001} &  0.047 {\tiny \color{gray} $\pm$ 0.001} &  0.060 {\tiny \color{gray} $\pm$ 0.005} \\
\cmidrule{2-8}
      & \multirow{2}{*}{$\mathcal{M}_{\tilde{\theta}}$} & single $\lambda$ &  0.208 {\tiny \color{gray} $\pm$ 0.021} &  0.108 {\tiny \color{gray} $\pm$ 0.001} &  0.095 {\tiny \color{gray} $\pm$ 0.002} &  0.085 {\tiny \color{gray} $\pm$ 0.002} &  0.161 {\tiny \color{gray} $\pm$ 0.011} \\
      &                                & multiple $\lambda$s &  0.114 {\tiny \color{gray} $\pm$ 0.016} &  0.095 {\tiny \color{gray} $\pm$ 0.001} &  0.118 {\tiny \color{gray} $\pm$ 0.002} &  0.119 {\tiny \color{gray} $\pm$ 0.003} &  0.178 {\tiny \color{gray} $\pm$ 0.003} \\
\bottomrule
\end{tabular}
\end{table}

\subsection{U-Net tomographic reconstruction}

We illustrate the tomographic reconstruction task in \cref{fig:tomographic-reconstruction-main} by showing the results obtained for one of the KMNIST images used in our experiments. For this example, we use layerwise regularisation parameters, following \cref{rec:factorised_prior}, but ablate \cref{rec:linear-weights}. That is, we compare the use of $\mcM_{\tilde \theta}$ and $\mcM_{\theta_{\star}}$. The former approach not only produces much larger pixel-wise uncertainties but also assigns uncertainty to large portions on the reconstructed image. Applying \cref{rec:linear-weights} yields smaller magnitude errorbars with the placement high uncertainty pixels matching the regions of increased reconstruction error much more closely. The histogram reveals the latter approach to result in much better calibrated uncertainty estimates.

\begin{figure}[h]
    \centering
    \includegraphics[width=\textwidth]{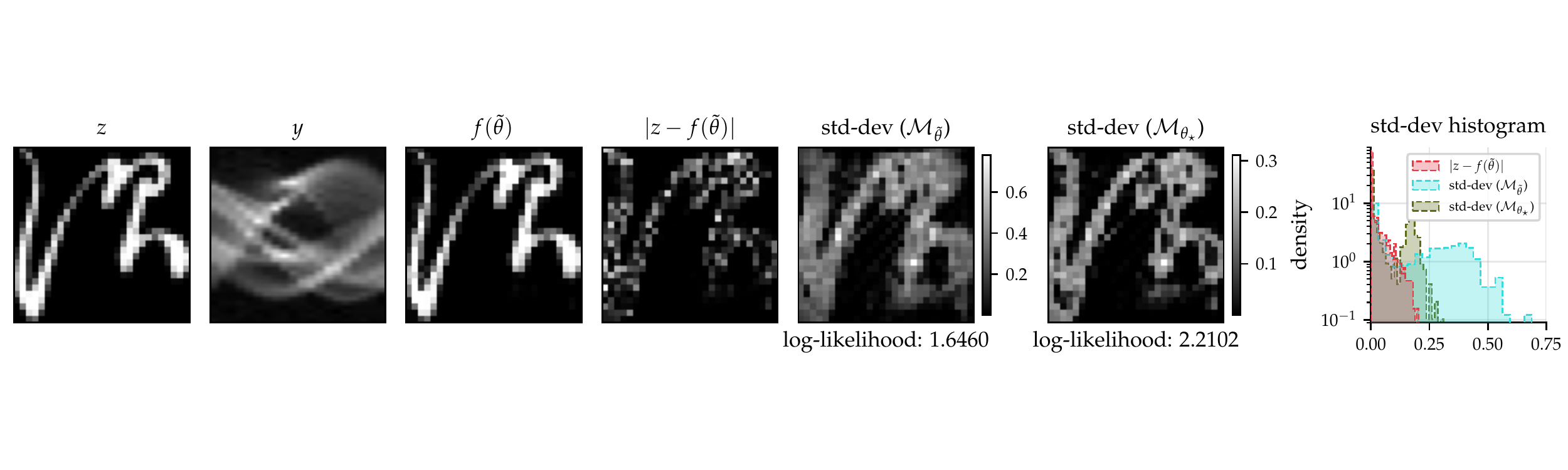}
    \vspace{-2cm}
    \caption{Enumerating the 7 plots from left to right: 1. Original KMNIST image. 2. Low dimensional projection obtained from applying \cref{eq:inverse_problem}. 3. Tomographic reconstruction from our U-net. 4. Associated pixel error in absolute value. 5. The pixel-wise marginal predictive standard deviation obtained when employing $\mcM_{\tilde \theta}$.  6. Same for $\mcM_{\theta_{\star}}$. 7. Histograms comparing reconstruction error with predictive standard deviation for both approaches under consideration.}
    \label{fig:tomographic-reconstruction-main}
\end{figure}

We further investigate the difference between the two objectives $\mcM_{\tilde \theta}$ and  $\mcM_{\theta_{\star}}$ in \cref{fig:DIP_prior_vars} by plotting the evolution of the block-wise regularisation strengths during successive iterations of these parameters' optimisation. The objective $\mcM_{\tilde \theta}$ results in most blocks' regularisation parameter $\lambda$ going to 0 (and thus the marginal prior variances diverging) for all 10 test images used in our experiments. This is due to the norm of the linearisation point $\tilde\theta$ being very large for the U-net. On the other hand, the  $\mcM_{ \theta_\star}$ objective is well behaved.

\begin{figure}[h]
    \centering
    \includegraphics[width=\textwidth]{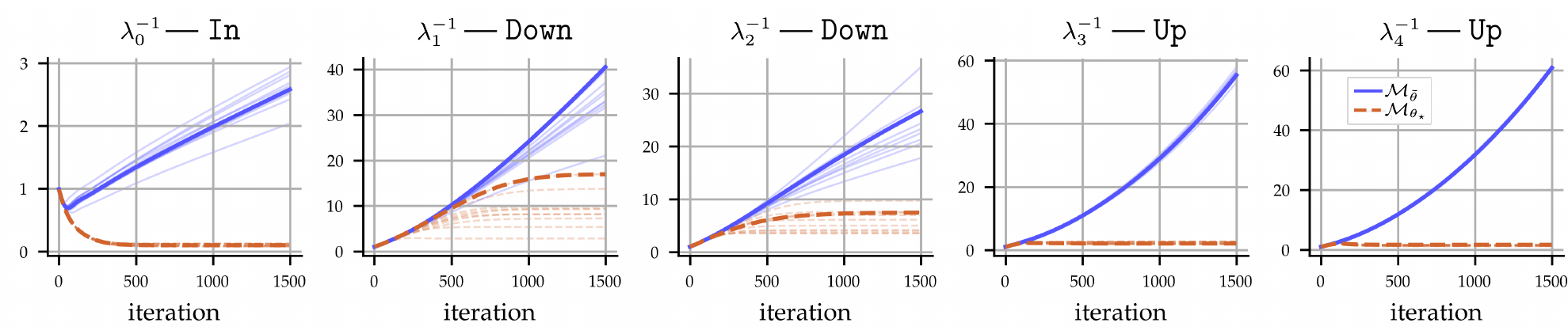}
    \caption{Hyperparameter optimisation traces for our U-net. Each figure corresponds to the regularisation parameter for a different convolutional block described in \cref{appendix:unet-setup}. The thick lines correspond to optimisation for the image shown in \cref{fig:tomographic-reconstruction-main}. Thin transparent lines correspond to the other 9 images used in our experiments. $\mcM_{\tilde \theta}$ leads hyperparameters to diverge, while $\mcM_{\theta_{\star}}$ converges.}
    \label{fig:DIP_prior_vars}
\end{figure}

\subsection{Image classification with full-scale ResNets}

In \cref{fig:resnet18_mnist} and \cref{fig:resnet50_cifar}, we plot the model evidence as well as test log-likelihoods (minus the NLL reported in \cref{sec:experiments}) obtained by the different models as a function of a single regularisation parameter $\lambda$, with $\Lambda = \lambda I$. \cref{fig:resnet18_mnist} considers ResNet-18 trained on MNIST, while \cref{fig:resnet50_cifar} considers ResNet-50 trained on CIFAR10. For both ResNets, we consider batch norm and FixUp .
We can see that in both settings, the optima of $\mcM_{ \theta_{\star}}$ corresponds to a larger $\lambda$ than that of $\mcM_{\tilde \theta}$. The former provides a higher test log-likelihood.
For the FixUp models, we further consider both the simplified and the full Taylor expanded linearised models (denoted $\mcM^{\star}_{\text{simple}}$ and $\mcM^{\star}_{\text{full}}$ respectively) for the purpose of finding the posterior mode $\theta_{\star}$ (recall that both linear models are equivalent for normalised networks). We see that the simple linear model finds larger $\lambda$ values than the Taylor expanded model and obtains better test log-likelihood.
Finally, \cref{tab:resnet18_mnist} is analogous to \cref{tab:resnet50_cifar} in that it compares the test log-likelihoods when using a single $\lambda$ value versus using a separate $\lambda$ values per layer, but reports results for ResNet-18 on MNIST. As in \cref{tab:resnet50_cifar}, we observe that using multiple $\lambda$ values generally yields better results.

\begin{figure}[htbp]
	\begin{minipage}{0.49\textwidth}
    	\centering
        \includegraphics[width=0.50\textwidth]{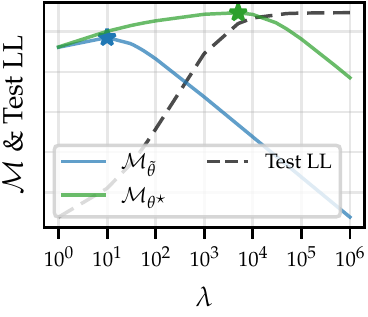}
        \includegraphics[width=0.485\textwidth]{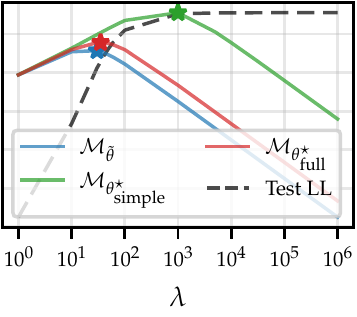}
		\captionof{figure}{ResNet-18 with BatchNorm (left) and FixUp (right) on MNIST. Model evidence and test log-likelihood as a function of $\lambda$.}
		\label{fig:resnet18_mnist}
    \end{minipage}\hfill
	\begin{minipage}{0.49\textwidth}
    	\centering
        \includegraphics[width=0.5\textwidth]{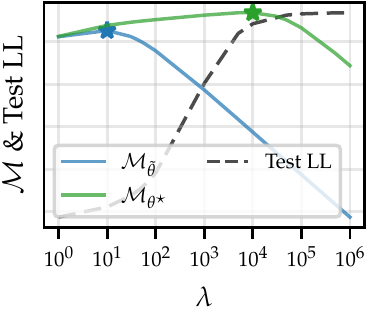}
        \includegraphics[width=0.485\textwidth]{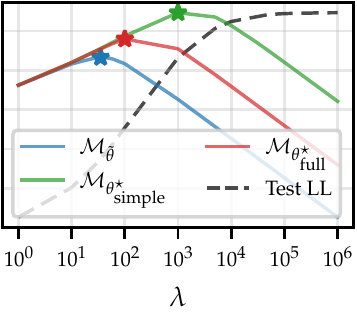}
    	\captionof{figure}{ResNet-50 with BatchNorm (left) and FixUp (right) on CIFAR10. Model evidence and test log-likelihood as a function of~$\lambda$.}
    	\label{fig:resnet50_cifar}
    \end{minipage}
\end{figure}

\begin{table*}[ht]
    \centering
    \caption{Test negative log-likelihoods for ResNet-18 on MNIST.}
    \label{tab:resnet18_mnist}
    \begin{tabular}{llll}
    \toprule
     & & \textsc{BatchNorm} & \textsc{FixUp}\\
    \midrule
    \multirow{2}{*}{$\mcM_{\theta_{\star, \text{simple}}}$} & single $\lambda$ & \textbf{0.003} {\tiny \color{gray} $\pm$ 0.000} & 0.007 {\tiny \color{gray} $\pm$ 0.001} \\
     & multiple $\lambda$s & \textbf{0.003} {\tiny \color{gray} $\pm$ 0.000} & \underline{\textbf{0.004}} {\tiny \color{gray} $\pm$ 0.000} \\
    \cmidrule{1-4}
    \multirow{2}{*}{$\mcM_{\theta_{\star, \text{full}}}$} & single $\lambda$ & 0.166 {\tiny \color{gray} $\pm$ 0.019} & 0.143 {\tiny \color{gray} $\pm$ 0.013} \\
     & multiple $\lambda$s & \underline{0.100} {\tiny \color{gray} $\pm$ 0.015} & \underline{0.100} {\tiny \color{gray} $\pm$ 0.012} \\
    \cmidrule{1-4}
    \multirow{2}{*}{$\mcM_{\tilde\theta}$} & single $\lambda$ & 0.529 {\tiny \color{gray} $\pm$ 0.009} & 0.288 {\tiny \color{gray} $\pm$ 0.015} \\
     & multiple $\lambda$s & 0.520 {\tiny \color{gray} $\pm$ 0.009} & \underline{0.237} {\tiny \color{gray} $\pm$ 0.017} \\
    \bottomrule
    \end{tabular}
\end{table*}

\end{document}